\documentclass[sigconf]{acmart}
\AtBeginDocument{%
	\providecommand\BibTeX{{%
			\normalfont B\kern-0.5em{\scshape i\kern-0.25em b}\kern-0.8em\TeX}}}
\setcopyright{iw3c2w3}
\copyrightyear{2021}
\acmYear{2021} 
\setcopyright{iw3c2w3}
\acmConference[WWW '21]{Proceedings of the Web Conference 2021}{April 19--23, 2021}{Ljubljana, Slovenia} 
\acmBooktitle{Proceedings of the Web Conference 2021 (WWW '21), April 19--23, 2021, Ljubljana, Slovenia}
\acmPrice{}
\acmDOI{10.1145/3442381.3450142}
\acmISBN{978-1-4503-8312-7/21/04}
\usepackage{subfigure} 
\usepackage{amsmath}
\usepackage{amsfonts}       
\usepackage{amsthm}
\usepackage{algorithm}
\usepackage{algorithmic}
\usepackage{array}
\usepackage{booktabs}       
\usepackage{bm}
\usepackage{caption}
\usepackage{color}
\usepackage{enumitem}
\usepackage{epstopdf}
\usepackage{float}
\usepackage{graphicx} 
\usepackage{multirow}
\usepackage{microtype}      
\usepackage{nicefrac}
\usepackage[normalem]{ulem}
\usepackage{pifont}
\usepackage[switch]{lineno} 
\usepackage{soul}
\usepackage{subfigure}
\usepackage{url}
\newcommand{\cmark}{\ding{51}}%
\newcommand{\xmark}{\ding{55}}%

\newcolumntype{L}[1]{>{\raggedright\let\newline\\\arraybackslash\hspace{0pt}}m{#1}}
\newcolumntype{C}[1]{>{\centering\let\newline  \\\arraybackslash\hspace{0pt}}m{#1}}
\newcolumntype{R}[1]{>{\raggedleft\let\newline \\\arraybackslash\hspace{0pt}}m{#1}}

\newtheorem{theorem}{Theorem}
\newtheorem{lemma}[theorem]{Lemma}
\newtheorem{prop}[theorem]{Proposition}
\newtheorem{corollary}[theorem]{Corollary}

\newtheorem{remark}{Remark}
\newtheorem{definition}{Definition}

\def \R{\mathbb R}

\def \st{\;\text{s.t.}\;}

\newcommand{\lonetwo}{\ell_{1\text{-}2}}
\newcommand{\bX}{\mathbf{X}}
\newcommand{\bU}{\mathbf{U}}

\newcommand{\bD}{\mathbf{D}}
\newcommand{\bA}{\mathbf{A}}
\newcommand{\bV}{\mathbf{V}}

\newcommand{\bZ}{\mathbf{Z}}
\newcommand{\bW}{\mathbf{W}}
\newcommand{\bH}{\mathbf{H}}
\newcommand{\bO}{\mathbf{O}}
\newcommand{\bQ}{\mathbf{Q}}

\newcommand{\be}{\mathbf{e}}
\newcommand{\bx}{\mathbf{x}}
\newcommand{\bb}{\mathbf{b}}

\newcommand{\bz}{\mathbf{z}}
\newcommand{\aff}{\mathcal{A}}

\newcommand{\bsig}{\bm{\sigma}}

\newcommand{\rnf}{r_{\text{NNFN}}}
\newcommand{\tilsig}{\tilde{\bsig}}
\newcommand{\NM}[2]{\| #1 \|_{#2} }  
 
\newcommand{\Prox}[2]{\text{prox}_{#1}(#2)}
\newcommand{\SO}[1]{\mathcal{P}_{\mathbf{\Omega}}(#1)}
\newcommand{\Tr}[1]{\text{tr}( #1 ) }

\newcommand{\Diag}[1]{\text{Diag}(#1)}

\usepackage{pdftexcmds}
\usepackage{catchfile}
\usepackage{ifluatex}
\usepackage{ifplatform}
\usepackage{balance}
\usepackage{afterpage}
\usepackage{bbding}
\ifwindows
\usepackage{times}
\fi
\iflinux
\fi
\ifmacosx
\fi 

\title
{A Scalable, Adaptive and Sound Nonconvex Regularizer 
	for Low-rank Matrix Learning}

\author{Yaqing Wang}
\affiliation{%
	\institution{Business Intelligence Lab, \\Baidu Research}
	\city{Beijing}
	\country{China}
}
\email{wangyaqing01@baidu.com}

\author{Quanming Yao}
\affiliation{%
	\institution{4Paradigm Inc.}
	\institution{EE, Tsinghua University}
	\city{Beijing}
	\country{China}
}
\email{qyaoaa@connect.ust.hk}

\author{James T. Kwok}
\affiliation{%
	\institution{Department of Computer Science, Hong Kong University of Science and Technology}
	\city{Hong Kong}
	\country{China}
}
\email{jamesk@cse.ust.hk}

\begin{document}

\begin{abstract}
Matrix learning is at the core of many machine learning problems.
A number of real-world applications such as collaborative filtering and text mining
 can be formulated as a low-rank matrix completion problems, which recovers incomplete matrix using low-rank assumptions.
To ensure that the matrix solution
has a low rank,
a recent trend is to use nonconvex regularizers that adaptively penalize singular values.
They offer good recovery performance and have nice theoretical properties, 
but are computationally expensive due to repeated access to individual singular values. 
In this paper, based on the key insight that adaptive shrinkage on singular values improve empirical performance, 
we propose a new nonconvex low-rank regularizer called "nuclear norm minus Frobenius norm" regularizer, which is scalable, adaptive and sound. 
We first show it provably holds the adaptive shrinkage property. 
Further, we discover its factored form which bypasses the computation of singular values and allows fast optimization by general optimization algorithms. 
Stable recovery and convergence are guaranteed. 
Extensive low-rank matrix completion experiments on a number of synthetic and
real-world data sets show that
the proposed method obtains state-of-the-art recovery performance while being the fastest in comparison to existing low-rank matrix learning methods. 
\footnote{Correspondence is to Q. Yao.}
\end{abstract}



\maketitle
	
\section{Introduction}
\label{sec:intro}

In many real-world scenarios, the data can be naturally represented as
matrices. Examples include the rating matrices in recommender systems
\cite{srebro2005maximum,koren2009matrix,candes2009exact,li2018adaerror,sharma2019adaptive}, 
the term-document
matrices  
of texts
in natural language processing
\cite{pennington2014glove,shi2018short}, 
images 
in computer vision 
\cite{hu2012fast,gu2014weighted}, and
climate observations 
in spatial-temporal analysis \cite{bahadori2014fast}. 
Thus, matrix learning is an important and fundamental tool in machine learning
\cite{srebro2005maximum,candes2008enhancing,chen2011integrating}, data mining
\cite{koren2009matrix,bahadori2014fast}, 
and computer vision
\cite{gu2014weighted,yao2019large}. 

In this paper, we focus on an important class of matrix learning problems, 
namely
matrix completion, which tries to predict the missing entries of a partially observed matrix \cite{candes2008enhancing}.
For example, in collaborative filtering \cite{koren2009matrix}, 
the rating matrix is often incomplete and one wants to predict the missing user ratings for all items. 
In climate analysis
\cite{bahadori2014fast}, 
observation records from 
only 
a few meteorological stations are available, and one wants to predict climate
information for the other locations.
In image inpainting \cite{hu2012fast,gu2014weighted},
the image has pixels missing and one wants to fill in these missing values. 
To avoid the
problem 
to be ill-posed, the target matrix is often assumed to have a 
low rank 
\cite{candes2009exact}. 
To obtain such a 
solution, a direct approach is to add a 
rank-minimizing term to the optimization objective.
However, rank minimization is NP-hard \cite{candes2009exact}. 
Thus, 
computationally, a 
more feasible approach is to 
use a regularizer that encourages the target matrix to have a small
rank.

There exist various low-rank regularizers. 
Nuclear norm regularizer is the tightest convex surrogate for matrix rank 
\cite{candes2009exact},  which has good recovery and convergence guarantees.
Defined as the sum of singular values,   
nuclear norm requires repeatedly computing the singular value decomposition (SVD), which is expensive. 
To be more efficient, a series of works instead turn to 
matrix factorization which factorizes the recovered matrix into factor matrices. Some of them work towards theoretical justification  \cite{tu2016low,wang2017unified}, while the other targets at designing better algorithms \cite{vandereycken2013low,boumal2015low,gunasekar2017implicit}. 
However, the performance of matrix factorization is not satisfactory \cite{fan2019factor,yao2019large}. 
To this end, factored low-rank regularizers are invented to balance efficiency and effectiveness,  
such as factored nuclear norm \cite{srebro2005maximum} and factored group-sparse regularizer (GSR) \cite{fan2019factor}. 
It is proved that factored nuclear norm can obtain comparable result 
as nuclear norm under mild condition
\cite{srebro2005maximum}.

Recently, nonconvex low-rank regularizers (Table~\ref{tab:reg_cmp}) which penalize less on the more informative large singular values 
are proposed, 
such as 
Schatten-p norm \cite{nie2012low}, 
truncated $\lonetwo$ norm \cite{ma2017truncated}, 
 capped-$\ell_1$ penalty \cite{zhang2010analysis},
log-sum penalty (LSP) \cite{candes2008enhancing}, and
minimax concave penalty (MCP) \cite{zhang2010nearly}. 
These nonconvex regularizers can outperform nuclear norm both theoretically \cite{gui2016towards,mazumder2020matrix} and empirically \cite{lu2015nonconvex,lu2015generalized,yao2019large}. 
However, as shown in Table~\ref{tab:reg_cmp}, 
\textit{none of the above-mentioned regularizers obtain (A) scalability, (B) good performance and (C-D) sound theoretical guarantee simultaneously}.  
 
\begin{table*}[ht]
	\centering
		\caption{Comparisons among nonconvex low-rank regularizers on 
		(A): Scalability (can be optimized in factored form); (B): Performance (can adaptively penalize singular values);
		(C): Statistical guarantee;
		(D): Convergence guarantee.}
	\begin{tabular}
	{c|c|c|c|c|c}
		\hline
		nonconvex low-rank regularizer & expression  & (A) & (B) & (C) & (D)
		\\ \hline
		factored nuclear norm \cite{srebro2005maximum}& $\min_{\bX = \bW\bH^\top}\frac{\lambda}{2} (\NM{\bW}{F}^2 + \NM{\bH}{F}^2)$  & \cmark&\xmark&\cmark&\cmark\\ \hline
		Schatten-p \cite{nie2012low}& $ \lambda (\sum_{i=1}^{m}\sigma^p_i(\bX))^{1/p}$ & \xmark&\xmark&\cmark&\cmark \\ \hline
		factored GSR \cite{fan2019factor}& $\min_{\bX = \bW\bH^\top}\frac{\lambda}{2} (\NM{\bW}{2,1} + \NM{\bH^\top}{2,1})$ & \cmark&\xmark&\cmark&\cmark  \\ \hline
		capped-$\ell_1$, LSP, and MCP \cite{lu2015generalized,yao2019large}& $\lambda \sum_{i=1}^{m}\hat{r}(\sigma_i(\bX))$ (see $\hat{r}$ in Table~\ref{tab:reg_form})  & \xmark&\cmark&\cmark&\cmark \\ \hline
		truncated $\lonetwo$ \cite{ma2017truncated}& $\sum\nolimits_{i = t+1}^n \sigma_{i}(\bX)
		- (\sum\nolimits_{i = t+1}^n \sigma^2_{i}(\bX))^{\nicefrac{1}{2}}$& \xmark&\cmark&\cmark&\cmark \\ \hline
		NNFN& $\NM{\bX}{*}-\NM{\bX}{F}$& \xmark&\cmark&\cmark&\cmark \\ \hline
		factored NNFN& $\min_{\bX = \bW\bH^\top}\frac{\lambda}{2}
		(\NM{\bW}{F}^2 \! + \! \NM{\bH}{F}^2) \! - \! \lambda\NM{ \bW\bH^\top}{F}$& {\cmark}&{\cmark}&{\cmark}&{\cmark}\\ \hline
	\end{tabular}
	\label{tab:reg_cmp}
\end{table*}  

To fill in this blank, we propose a scalable, adaptive and sound nonconvex regularizer based on the key insight that adaptive shrinkage property of common nonconvex regularizers can improve empirical performance. 
Specifically, 
Our contribution can be summarized as follows:
\begin{itemize}[leftmargin=*]
	\item We propose a new nonconvex regularizer called "nuclear norm minus Frobenius norm" (NNFN) regularizer for low-rank matrix learning, which is scalable, adaptive and theoretically guaranteed. 
	\item 
	We show that NNFN regularizer can be factorized to sidestep the expensive SVD. 
	This problem can be optimized by general algorithms such as gradient descent. 
	\item We provide sound theoretical analysis on statistical and convergence properties of both NNFN and factored NNFN regularizers. 
	\item We conduct extensive experiments on 
	both synthetic and a number of real-world data sets including  
	recommendation data 
	and climate record data.
	In comparison to existing methods, 
	results consistently show that the proposed algorithm obtains state-of-the-art
	recovery performance while being the fastest.
\end{itemize}

\noindent
{\bf Notations}: 
Vectors are denoted by lowercase
boldface, matrices by uppercase boldface. 
$(\cdot)^\top$ denotes transpose operation and  $\mathbf{A}_{+} \! = \! [\max (A_{ij}, 0)]$. 
For a vector $\bx = [x_i] \in \R^{m}$, 
$\Diag{\bx}$ constructs a $m \times m$ diagonal matrix with the $i$th diagonal element being  $x_i$.  
$\mathbf{I}$ denotes the identity matrix.
For a 
square matrix $\bX$,
$\Tr{\bX}$ is its trace.
For matrix $\bX\in\R^{m\times n}$ (without loss of generality, we assume that $m\ge n$),
$\NM{\bX}{F} = \sqrt{\Tr{\bX^{\top} \bX}}$ is its Frobenius norm.
Let the singular value decomposition (SVD) of a rank-$k^*$ $\bX$ be
$\bU\Diag{\bsig(\bX)}\bV^\top$, where $\bU\in\R^{m\times k^*}$, $\bV\in\R^{n\times k^*}$,
$\bsig(\bX)=[\sigma_i(\bX)]\in\R^{k^*}$ with 
$\sigma_i(\bX)$ being the $i$th singular value of $\bX$ and 
$\sigma_1(\bX)\ge\sigma_2(\bX)\ge\dots \ge \sigma_k(\bX)\ge0$.

\section{Background: Low-Rank Matrix Learning}
\label{sec:back_low_rank}

As minimizing the rank is NP-hard \cite{candes2009exact}, 
low-rank matrix learning is often formulated as
the following 
optimization
problem: 
\begin{equation}
\min\nolimits_{\bX} f(\bX) 
+ 
\lambda 
r(\bX), 
\label{eq:mc_reg}
\end{equation}
where $f$ is a smooth 
function
(usually the loss), $r(\bX)$ is a regularizer that encourages
$\bX$ 
to be low-rank,
and $\lambda \ge 0$ is a tradeoff hyperparameter. 
Let $\mathbf{\Omega} \in \{0,1\}^{m \times n}$ record
positions of the observed entries (with
$\Omega_{ij}=1$ if $O_{ij}$ is observed, and $0$ otherwise), and
$\SO{\cdot}$ is a projection operator such that
$[\SO{\mathbf{A}}]_{ij} = A_{ij}$
if  $\Omega_{ij} = 1$ and $0$ otherwise. 
Low-rank matrix completion \cite{candes2009exact} tries to recover the underlying  low-rank matrix  $\bX\in \R^{m \times n}$ from an incomplete matrix $\bO \in \R^{m \times n}$ with only a few observed entries. 
It 
usually sets $f(\bX)$ as
\begin{align}
	f(\bX)\equiv\frac{1}{2}\NM{\SO{\bX - \bO}}{F}^2, 
\end{align}
which measures the recovery error.

\subsection{Convex Nuclear Norm Regularizer} 

The convex nuclear norm
$\NM{\bX}{*}=\NM{\bsig(\bX)}{1}$ \cite{candes2009exact}, 
 is the tightest convex surrogate of the matrix rank \cite{fazel2002matrix}.
Problem~(\ref{eq:mc_reg}) 
is usually solved by the  proximal algorithm 
\cite{parikh2014proximal}.
At the $t$th iteration, it generates the next iterate by computing the proximal
step
$\bX_{t + 1} = \Prox{ \eta \lambda r }{ \bX_t - \eta\lambda \nabla (\bX_t) }$,
where 
$\eta > 0$ is the stepsize, and
$\Prox{ \lambda r }{ \bZ } = \arg\min_{\bX} \frac{1}{2}\NM{\bX - \bZ}{2}^2 + \lambda r(\bX)$
is the proximal operator. In general,
the proximal operator should be easily computed.
For the nuclear norm, its
proximal operator is computed as
\cite{cai2010singular}:
\begin{equation} \label{eq:prox}
\Prox{\lambda \NM{\cdot}{*}}{\bZ} 
= \bU \left( \Diag{\bsig(\bZ)} - \lambda \mathbf{I}\right)_{+} \bV^{\top},
\end{equation} 
where 
$\bU \Diag{\bsig(\bZ)} \bV^{\top}$ is the SVD of $\bZ$. 

\subsection{Nonconvex Regularizers}

Recently, various nonconvex regularizers appear (Table~\ref{tab:reg_cmp}). 
Common examples include
the capped-$\ell_1$ penalty \cite{zhang2010analysis}, log-sum penalty
(LSP) \cite{candes2008enhancing}, and minimax concave penalty (MCP)
\cite{zhang2010nearly}. 
\begin{table}[ht]
	\centering
	\caption{Nonconvex low-rank regularizers in the form of \eqref{eq:noncvx}.}
	\begin{tabular}
		{c| c} 
		\hline
		& $\hat{r}(\sigma_i(\bX))$
		\\ \hline
		capped-$\ell_1$  & $\min(\sigma_i(\bX), \theta )$  \\ \hline
		LSP& $\log \left(\frac{1}{\theta} \sigma_i(\bX) + 1\right)$  \\ \hline
		MCP & 
		$\begin{cases}
		\sigma_i(\bX) - \frac{\sigma_i^2(\bX)}{2 \theta\lambda} &
		\text{if}\; \sigma_i(\bX) \le \theta \lambda\\
		\frac{\theta \lambda}{2}             & \text{otherwise}
		\end{cases}$
		\\ \hline
	\end{tabular}
	\label{tab:reg_form}
\end{table}
As shown in Table~\ref{tab:reg_form}, 
they can be written in the general form of
\begin{equation}\label{eq:noncvx}
r(\bX) 
= \sum\nolimits_{i = 1}^n \hat{r}(\sigma_i(\bX)), 
\end{equation}
where $\hat{r}(\alpha)$ is nonlinear, 
concave and non-decreasing for $\alpha \ge 0$
with $\hat{r}(0) = 0$. 
In contrast to the proximal operator 
in 
(\ref{eq:prox}) which
penalizes all singular values of $\bZ$ by the same amount $\lambda$, these nonconvex regularizers penalize less on 
the larger 
singular values which are 
more informative. 
Additionally, 
the nonconvex 
Schatten-p norm \cite{nie2012low} can better approximate rank than nuclear norm. 
Truncated $\lonetwo$ regularizer \cite{ma2017truncated}  
can obtain unbiased approximation for rank 
\cite{ma2017truncated}.  
These nonconvex regularizers outperform nuclear norm
on many applications 
empirically \cite{gu2014weighted,lu2015generalized,yao2019large}, and can obtain 
lower recovery errors 
\cite{gui2016towards}. 
However, learning with 
nonconvex regularizers is very difficult. 
It usually requires dedicated
solvers to leverage special structures
(such as the low-rank-plus-sparse structure in 
\cite{hastie2015matrix,yao2019large}) or involves several iterative algorithms (such as  the difference of convex
functions algorithm (DCA) \cite{hiriarturruty1985generalized} with subproblems solved by the alternating direction method of multipliers
(ADMM) \cite{boyd2011distributed} for truncated $\lonetwo$ regularized problem).
This computation bottleneck limits their applications in practice. 

\subsection{Factored Regularizers}
Note that aforementioned regularizers require access to individual singular values.
As computing the singular values of a $m\times n$ matrix (with
 $m\ge n$) via SVD takes $O(mn^2)$
time, this can be costly for a large matrix.  
Even when rank-$k$ truncated SVD is used, the computation cost is  still
$O(mnk)$.
To relieve the computational burden, factored low-rank regualrizers are invented.
\eqref{eq:mc_reg} can then be rewritten into a factored form as   
\begin{equation}\label{eq:mc_fact}
\min\nolimits_{\bW,\bH}
f(\bW\bH^\top) + \mu g(\bW,\bH),
\end{equation}
where $\bX$ is factorized into $\bW \in \R^{m \times k}$ and $\bH \in \R^{n \times k}$, 
and 
$\mu\geq 0$ is a hyperparameter. 
When $\mu=0$, this reduces to matrix factorization \cite{vandereycken2013low,boumal2015low,tu2016low,wang2017unified,gunasekar2017implicit}.
Not all regularizers $r(\bX)$ have equivalent factored form $g(\bW,\bH)$. 
For a matrix $\bX$ with rank $k^*\leq k$, it is already discovered that 
nuclear norm can be
rewritten in a factored form \cite{srebro2005maximum} as 
$\NM{\bX}{*} = \min_{\bX = \bW\bH^\top} \nicefrac{1}{2} (\NM{\bW}{F}^2 + \NM{\bH}{F}^2)$. 
As for nonconvex low-rank regularizers, 
only Schatten-p norm can be approximated by factored forms \cite{shang2016tractable,fan2019factor}. 
Other nonconvex regularizers, which 
need to penalize individual singular values, 
cannot be written in factored form. 

\section{Nuclear Norm Minus Frobenius Norm (NNFN) Regularizer}
\label{sec:proposed}

Based on the insight that adaptive shrinkage on singular values can improve empirical performance, 
we present 
a new nonconvex regularizer
\begin{equation} \label{eq:new}
\rnf(\bX) 
= \NM{\bX}{*}-\NM{\bX}{F}, 
\end{equation}
which will be called the ``nuclear norm minus Frobenius norm" (NNFN) regularizer.  
Next, we will show that NNFN regularizer 
applies adaptive shrinkage for singular values provably,  
has factored form which allows fast optimization by general algorithms, 
and has sound theoretically guarantee.  

\subsection{Adaptive Shrinkage Property}
\label{sec:prop}

Recall from (\ref{eq:prox}) that the proximal operator of the nuclear norm
equally penalizes each singular value by $\lambda$ until it reaches zero.
In contrast, 
we find that common noncovex regularizers $r(\bX)$ of the general form \eqref{eq:noncvx} all hold the adaptive shrinkage property in 
Proposition~\ref{pr:property}\footnote{All the proofs are in Appendix~\ref{app:proof}.}. 
\begin{prop}[Adaptive Shrinkage Property] \label{pr:property}
	Let $r(\bX)$ be a nonconvex low-rank regularizer of the form \eqref{eq:noncvx}, 
	and 
	$\tilsig = [\tilde{\sigma}_i] =\Prox{\lambda r(\cdot)}{\bsig(\bZ)}$ in (\ref{eq:prox_nnfn}).
	Then,
	(i) ${\sigma}_i(\bZ) \! \ge \! \tilde{\sigma}_i$ (shrinkage);
	and
	(ii) $\sigma_i(\bZ) \! - \! \tilde{\sigma}_i  \!\le \! \sigma_{i + 1}(\bZ) \! - \! \tilde{\sigma}_{i+1}$ 
	(adaptivity),
	where strict inequality holds at least for one $i$.
\end{prop}
It shows that $\Prox{\lambda r}{\cdot}$
adaptively shrinks the singular values of its matrix argument, 
in that larger singular values are penalized less.
This property is important for obtaining good empirical
performance 
\cite{gu2014weighted,hu2012fast,lu2015generalized,lu2015nonconvex,yao2019large}. Other nonconvex regularizers such as truncated $\lonetwo$ and Schatten-p norm, do not have this property due to the lack of analytic proximal operators.
Figure~\ref{fig:reg_prox}
shows the shrinkage performed by adaptive nonconvex regularizers versus the convex nuclear norm regularizer. 
As can be seen,  
the convex nuclear norm regularizer
shrinks all singular values by the same amount; whereas the adaptive
nonconvex regularizers enforce different amounts of 
shrinkage depending on the magnitude of
$\sigma_i(\bZ)$. 

\begin{figure}[ht]
	\centering
	\includegraphics[width = 0.246\textwidth]{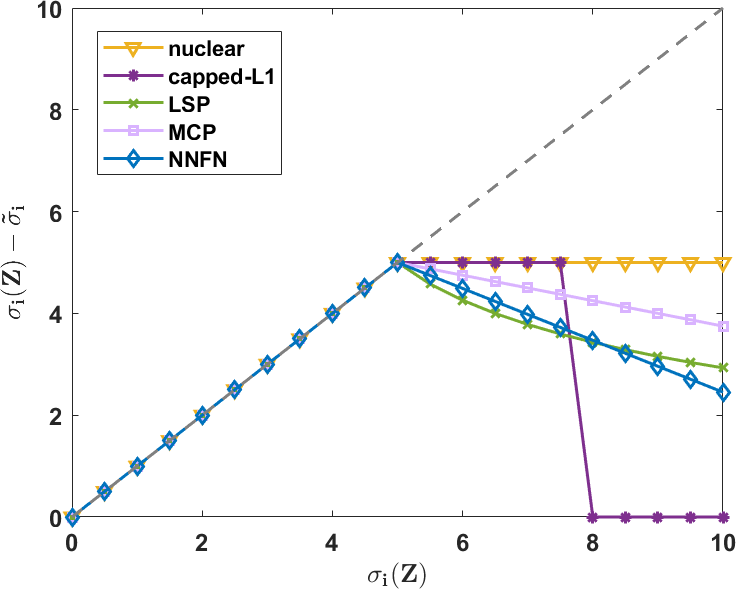}	
	\caption{Shrinkage 
		performed by different regularizers.
		The hyperparameters are tuned such that $\tilde{\sigma}_i$ is zero for $\sigma_i(\bZ)\le 5$.  }
	\label{fig:reg_prox}
\end{figure}

Here, we show
the proposed NNFN regularizer in \eqref{eq:new} also provably satisfies  
adaptive shrinkage of the
singular values when used with a proximal algorithm. 
We first present the proximal operator of $\rnf(\cdot)$ in Proposition~\ref{pr:l12_to_nnfn}. As $\Prox{\lambda  \NM{\cdot}{1\text{-}2}}{\bsig(\bZ)}$ returns a sparse vector \cite{lou2018fast}, the resultant $\Prox{\lambda \rnf}{\bZ}$ is low-rank.  

\begin{prop}\label{pr:l12_to_nnfn}
Given a matrix $\bZ$, 
let its SVD be  $\bar{\bU}\Diag{{\bsig}(\bZ)}\bar{\bV}^\top$, 
and $\lambda\le \NM{\bsig(\bZ)}{\infty}$.
\begin{align}
\Prox{\lambda \rnf}{\bZ}
= \bar{\bU}\Diag{\Prox{\lambda \NM{\cdot}{1\text{-}2}}{\bsig(\bZ)}}
\bar{\bV}^\top,
\label{eq:prox_nnfn}
\end{align}
where 
$\Prox{\lambda  \NM{\cdot}{1\text{-}2}}{\bz}$ has closed-form solution \cite{lou2018fast}.
\end{prop} 

Now, we are ready to prove in 
the following
Corollary that NNFN regularizer also shares the adaptive shrinkage property. This can lead to better empirical performance as discussed earlier.  
\begin{corollary} \label{cr:property_other} 
	The two properties in Proposition~\ref{pr:property} also hold for
	the proximal operators of the NNFN regularizer.
\end{corollary}

\section{Algorithms for (\ref{eq:mc_reg}) with NNFN Regularizer}
\label{sec:opt}

With the proximal operator obtained in 
Proposition~\ref{pr:l12_to_nnfn},
learning with the NNFN regularizer can be readily solved with the proximal algorithm.
However, it
still relies on computing the SVD in each iteration. 
To tackle this problem, 
we then present a simple and scalable algorithm that avoids SVD computations by using the factored NNFN regularizer .

\subsection{A Proximal Algorithm}
\label{sec:naive}

We first present a direct application of the proximal algorithm to
problem \eqref{eq:mc_reg} with the NNFN regularizer.
At the $t$th iteration, 
we obtain  $\bZ^{t}=\bX^{t-1} - \eta\nabla f(\bX^{t-1})$,
and then perform the proximal step in
Proposition~\ref{pr:l12_to_nnfn}.
The complete procedure 
is shown in  Algorithm~\ref{alg:nnfn_pa}. 

\begin{algorithm}[H]
	\caption{A proximal algorithm for \eqref{eq:mc_reg} with NNFN.}
	\begin{algorithmic}[1]
		\REQUIRE Randomly initialized $\bX^0$, stepsize $\eta$;
		\FOR{$t=1,\dots,T$}
		\STATE obtain $\bZ^{t}=\bX^{t-1} - \eta\nabla f(\bX^{t-1})$; 
		\STATE update $\bX^{t}$ as $\Prox{\lambda \rnf}{\bZ^t}$;
		\ENDFOR
		\RETURN $\bX^T$.
	\end{algorithmic}
	\label{alg:nnfn_pa}
\end{algorithm}

\subsubsection{Complexity}  
The iteration time complexity of 
Algorithm~\ref{alg:nnfn_pa} is dominated by SVD. 
Let $r_t$ ($n \geq r_t\geq k$) be the rank estimated at the $t$th iteration. We can perform rank-$k$ truncated SVD, which 
takes $O(mnk)$. 
The space complexity is $O(mn)$ to keep full matrices.

\subsection{A General Solver for Factored Form}
\label{sec:better}

In this section, we propose 
a more efficient solver which removes the SVD bottleneck.
The key observation is
that the NNFN regularizer 
in (\ref{eq:new}) 
can be computed on the recovered matrix without touching singular values explicitly. 
The 
Frobenius norm 
of a matrix
can be computed 
without using its
singular values, and the 
nuclear norm can be replaced by the factored nuclear norm. 
With this factored NNFN regularizer, 
the matrix learning problem then becomes:
\begin{equation}\label{eq:fnnfn}
\min_{\bW,\bH}F(\bW,\bH)\equiv f(\bW\bH^\top)
+ \frac{\lambda}{2}
\left( \NM{\bW}{F}^2+\NM{\bH}{F}^2 \right) 
- \lambda\NM{ \bW\bH^\top}{F}.
\end{equation}
Thus, SVD can be completely avoided.
In contrast, the other nonconvex low-rank regularizers (including the very related
truncated $\lonetwo$ regularizer with $t\neq 0$)
need to penalize individual singular values, and so
do not have factored form.




Unlike other regualrizers which requires dedicated solvers, the reformulated 
problem \eqref{eq:fnnfn} can be simply solved by general solvers such as gradient descent. In
particular,
gradients of $F(\bW,\bH)$ can be easily obtained.
Let $\bQ \equiv \bW\bH^{\top} \not= \mathbf{0}$, and
	$c =  \lambda /\NM{ \bW\bH^\top}{F}$. Then, we obtain
	\begin{align}\label{eq:update-w}
		\nabla_{\bW} F(\bW,\bH)
		&=[\nabla_{\bQ} f(\bQ)]\bH\!+\!\lambda \bW\! -\! c \bW (\bH^\top \bH),\\\label{eq:update-h}
		\nabla_{\bH} F(\bW,\bH)
		&= [\nabla_{\bQ} f(\bQ)]^\top \bW\!+\!\lambda \bH \!-\! c \bH (\bW^\top \bW). 
	\end{align}
These only involve
simple matrix multiplications, without any SVD computation.
Moreover, we can easily replace the simple gradient descent by recent solvers with improved performance.
The complete procedure is shown in Algorithm~\ref{alg:nnfn_gd}. 
\begin{algorithm}[H]
	\caption{A general solver for \eqref{eq:mc_reg} with factored NNFN.}
	\begin{algorithmic}[1]
		\REQUIRE  Randomly initialized $\bW^0, \bH^0$, stepsize $\eta$;
		\FOR{$t = 1, \dots, T $}
		\STATE update $\bW^{t} = \bW^{t-1}-\eta\nabla_{\bW} F(\bW^t,\bH^t)$ using \eqref{eq:update-w};
		\STATE update $\bH^{t} = \bH^{t-1}-\eta\nabla_{\bH} F(\bW^t,\bH^t)$ using \eqref{eq:update-h};
		\ENDFOR 
		\RETURN $\bX^T = \bW^{T}(\bH^{T})^\top$.
	\end{algorithmic}
	\label{alg:nnfn_gd}
\end{algorithm}

\subsubsection{Complexity}  
Learning with factored NNFN does not need the expensive SVD, thus it has a much lower time
complexity. Specifically, 
multiplication of the sparse matrix 
	$\nabla_{\bQ} f(\bQ)=
\SO{\bQ - \bO}$
and 
$\bH$ in 
	$\nabla_{\bW} F(\bW,\bH)$
(and similarly 
multiplication of 
$[\nabla_{\bQ} f(\bQ)]^\top$
and $\bW$ in
	$\nabla_{\bH} F(\bW,\bH)$)
takes $O( \NM{\mathbf{\Omega}}{0} k)$ time, 
computation of
$\bW(\bH^\top \bH)$, $\bH(\bW^\top \bW)$ and  
$\NM{\bW\bH^\top}{F}$ (computed as $\sqrt{\Tr{(\bH^\top \bH)(\bW^\top \bW)}}$)
takes $O(mk^2)$ time. 
Thus, the 
iteration
time complexity
is $O(\NM{\mathbf{\Omega}}{0}k+ mk^2)$. 
As for space, using the factored form reduces the parameter size from
$O(mn+\NM{\mathbf{\Omega}}{0})$ to $O(mk+\NM{\mathbf{\Omega}}{0})$, where
$\NM{\mathbf{\Omega}}{0}$ is the space for keeping 
a sparse 
$\bO$. 

\begin{table*}[hbtp]
	\centering
		\caption{State-of-the-art solvers for various matrix completion methods.  
		Here, $r_t$ (usually $\geq k$) is an estimated rank at the $t$th iteration,
		$\hat{r}_t=r_t+r_{t-1}$, and $q$ is number  of inner ADMM iterations used in
		\cite{ma2017truncated}.
	}
	\begin{tabular}
		{C{135px} |C{180px}|c|c}
		\hline
		regularizer&state-of-the-art solver&             time
		complexity      &             space
		complexity                    \\ \hline
		nuclear norm \cite{candes2009exact} & softimpute algorithm with alternating least squares\cite{hastie2015matrix}              &                   $O(\NM{\mathbf{\Omega}}{0} k + m\hat{r}_t^2)$            &$O((m + n)r_t + \NM{\mathbf{\Omega}}{0})$                        \\\hline 
		factored nuclear norm\cite{srebro2005maximum}	& alternating gradient descent \cite{ge2016matrix}     &   $O(\NM{\mathbf{\Omega}}{0} k + mk)$    &$O((m + n)k + \NM{\mathbf{\Omega}}{0})$ \\ \hline 
		probabilistic matrix  factorization \cite{mnih2008probabilistic} &	Bayesian probabilistic matrix factorization  solver using Markov Chain Monte Carlo  \cite{salakhutdinov2008bayesian}      &   $O(\NM{\mathbf{\Omega}}{0} k^2 + mk^3)$  &$O((m + n)k + \NM{\mathbf{\Omega}}{0})$   \\ \hline
		factored GSR \cite{fan2019factor}& proximal alternating linearized algorithm coupled with iteratively reweighted minimization \cite{fan2019factor} & $O(mnk)$&$O((m + n)k + \NM{\mathbf{\Omega}}{0})$\\ \hline
		truncated $\lonetwo$ \cite{ma2017truncated} & DCA algorithm with sub-problems solved by ADMM algorithm\cite{ma2017truncated}&$O(qmn^2)$&$O(mn)$\\\hline
		capped-$\ell_1$, LSP, and MCP \cite{lu2015generalized,yao2019large}& a solver leveraging power method and  "low-rank plus sparse" structure \cite{yao2019large} 
		& $O(\NM{\mathbf{\Omega}}{0} r_t + m\hat{r}^2_t)$ &$O((m + n)r_t + \NM{\mathbf{\Omega}}{0})$ \\ \hline
		NNFN&proximal algorithm                                 &
		$O(mnr_t)$                         &$O(mn)$         \\\hline
		factored NNFN&general solvers such as gradient descent                           &   $O(\NM{\mathbf{\Omega}}{0} k + mk^2)$  &$O((m + n)k + \NM{\mathbf{\Omega}}{0})$  \\ \hline
	\end{tabular}
	\label{tab:cost}
\end{table*}

\subsection{Comparison with Optimizing Other Regularizers} 
We compare the proposed solvers with state-of-the-art solvers for other regularizers 
	in Table~\ref{tab:cost}.  	
Among nonconvex regularizers, only
factored NNFN can be solved by general solvers such as gradient descent. which makes it simple and efficient. 
In contrast, other  
nonconvex regularizers are difficult to optimize and require dedicated solvers. 
Although the time complexity is comparable in big O, we observe in experiments that learning with factored NNFN  is much more scalable.
	Additionally,  for space,
only solvers for truncated $\lonetwo$ \cite{ma2017truncated} and NNFN require
keeping the complete matrix
which takes $O(mn)$ space,
while the other methods have comparable and much smaller space requirements. 

\begin{remark} 
	Truncated $\lonetwo$ regularizer
	\cite{ma2017truncated} is a related existing nonconvex regularizer. 
	When $t=0$, it reduces to NNFN regularizer. 
	However, 
	without the operation to truncate singular values, NNFN regularizer 
	(1) is proved to enforce adaptive shrinkage while truncated $\lonetwo$ does not;
	(2) allows cheap closed-form proximal operator while truncated $\lonetwo$ requires a 
	combined use of DCA and ADMM; 
	(3) can be efficiently optimized in factored form without taking SVD while truncated $\lonetwo$ can not;  
	(4) has recovery bound for both itself and its factored form while the analysis in \cite{ma2017truncated} does not apply for factored form.  
	Therefore, the discovery of NNFN regularizer is new and important. 
\end{remark}
\section{Theoretical Analysis} 
\label{sec:theory}

Here, we analyze the statistical and convergence properties for the proposed algorithms.

\subsection{Recovery Guarantee}

We establish statistical guarantee based on Restricted Isometry Property (RIP) \cite{candes2005decoding} introduced below.
\begin{definition}[Restricted Isometry Property (RIP) \cite{candes2005decoding}] 
	An affine transformation $\mathcal{A}$ satisfies RIP if for all $\bX \in
	\R^{m\times n}$ of rank at most $k$, 
	there exists a constant $\delta_k\in(0,1)$ such that: 
	\begin{equation}\label{eq:rip}
	(1- \delta_k) \NM{\bX}{F}^2 \leq \|\mathcal{A}(\bX)\|_2^2 \leq (1+\delta_k) \NM{\bX}{F}^2.
	\end{equation}
\end{definition}

Under the RIP condition, 
we prove in the following that stable recovery is guaranteed where
the estimation error depends linearly on $\NM{\be}{2}^2$. 
\begin{theorem}[Stable Recovery]\label{thm:recovery_error}
	Consider $f(\bX) =\frac{1}{2}\NM{\aff(\bX)-\bb}{2}^2$, where
	$\aff$ is an affine transform
	satisfying the 
	RIP with $\delta_{2k}
	\leq 1/3$, 
	and $\bb = \aff(\bX^*) + \be$ is a measurement vector corresponding to a rank-$k^*$ matrix $\bX^*$ and error vector $\be$. 
	Assume sequence $\{\bX^t\}$ with $f(\bX^{t+1})< f(\bX^{t})$ and each $\bX^t$ is the iterate obtained by optimizing the following two 
	equivalent constrained formulations 
	of \eqref{eq:fnnfn}: 
	(i) $\bX^t$ is the iterate of optimizing $\min\nolimits_{\bX} 
		f(\bX)
		\!\;\!\st\!\;\!
		\rnf(\bX) \!\le\! \beta'$, where $\beta'\geq 0$ is a hyperparameter. 
	or (ii)  $\bX^t\!=\!\bW^t(\bH^t)^\top$ is the iterate of optimizing $\min\nolimits_{\bW,\bH} 
		f(\bW\bH^\top)$ $
		\!\;\!\st\!\;\! 
		\frac{1}{2}(\NM{\bW}{F}^2\!+\!\NM{\bH}{F}^2) \!-\! \NM{ \bW\bH^\top }{F} \!\le\! \beta'$, where $\beta'\geq 0$ is a hyperparameter. 
	Then, 
	the recovery error 
	is bounded as $\NM{\bX^t-\bX^*}{F}^2\le c\|\be\|^2_2$  
	for some constant $c$ and sufficiently large $t$. 
\end{theorem}
Existing theoretical analysis \cite{gui2016towards} applies for adaptive nonconvex regularizer with separable penalty on individual singular values. Hen it 
does not apply for NNFN regualrizer which is not separable.

\subsection{Convergence Guarantee}

The proximal algorithm for NNFN regularizer is guaranteed to converge to critical points \cite{bolte2014proximal}.
As for the non-smooth factored NNFN regularizer, 
the following guarantee convergence to a critical point of
\eqref{eq:fnnfn}, which can be used to form a critical point of the original low-rank
matrix completion problem in \eqref{eq:mc_reg}.
\begin{theorem}[Convergence Guarantee]\label{pr:convergence}
	Assume that $k$ is sufficiently large and $\bW^t (\bH^t)^{\top} \not= \mathbf{0}$ during the iterations, 	
	gradient descent on \eqref{eq:fnnfn} can converge to a critical point $(\bar{\bW},\bar{\bH})$.  
	Moreover, the obtained $\bar{\bX}=\bar{\bW}\bar{\bH}^\top$ is also a critical point of \eqref{eq:mc_reg},
	with $r$ being the NNFN regularizer.
\end{theorem}


\section{Experiments}
\label{sec:expts} 
Here, we
perform 
matrix completion
experiments on 
both 
synthetic and 
real-world data sets,
using a PC with Intel i7 3.6GHz CPU and 48GB memory.
Experiments
are repeated five times,
and the  averaged 
performance are 
reported. 

\subsection{Experimental Settings}
\subsubsection{Baselines}
The proposed \textbf{NNFN}\footnote{Our codes are available at \url{https://github.com/tata1661/NNFN}} regularizer solved by proximal algorithm, 
and its
scalable variant \textbf{factored NNFN} solved by gradient descent, 
are compared with the  representative regularizers optimized by their respective state-of-the-art solvers as listed in Table~\ref{tab:cost}. 
For all methods that we compare in the experiments, we use public codes unless they are not available. 
	\begin{itemize}
	\item Low-rank regularizers include: 
	(i) \textbf{nuclear} norm\footnote{\url{https://cran.r-project.org/src/contrib/softImpute_1.4.tar.gz}, we rewrite it in MATLAB} \cite{hastie2015matrix}  
	\cite{candes2009exact}; 
	(ii)
	\textbf{truncated $\lonetwo$} regularizer\footnote{\url{https://sites.google.com/site/louyifei/TL12-webcode.zip?attredirects=0&d=1}}  
	\cite{ma2017truncated};  
	(iii) adaptive nonconvex low-rank regularizers of the form \eqref{eq:noncvx}\footnote{\url{https://github.com/quanmingyao/FaNCL}},
	including the \textbf{capped-$\ell_1$} penalty \cite{zhang2010analysis},
	\textbf{LSP} \cite{candes2008enhancing}; and
	\textbf{MCP} \cite{zhang2010nearly}. 
	
	\item Factored regularizers include 
	(i) \textbf{factored nuclear} norm\footnote{We implement it on our own.} \cite{srebro2005maximum}; 
	(ii) \textbf{BPMF}\footnote{\url{https://www.cs.toronto.edu/~rsalakhu/BPMF.html}} \cite{mnih2008probabilistic}; 
	and (iii) \textbf{factored GSR}\footnote{\url{https://github.com/udellgroup/Codes-of-FGSR-for-effecient-low-rank-matrix-recovery}} \cite{fan2019factor}. 
\end{itemize}

\begin{table*}[htb]
	\centering
	\caption{Performance on the synthetic data $\bO\in\R^{m\times m}$ with different $m$'s.  
		For each data set, its sparsity ratio is shown in brackets. 
		The best and comparable results (according to the pairwise t-test with 95\% confidence) are highlighted in bold.
	}
	\begin{tabular}{c|cc|cc|cc}
		\hline
		&     \multicolumn{2}{c|}{$m=500$ (12.43\%)}
		&\multicolumn{2}{c|}{$m=1000$ (6.91\%)}&      \multicolumn{2}{c}{$m=2000$ (3.80\%)}    \\
		&          testing NMSE &         time (s)          &         testing NMSE           &
		time (s)          &         testing NMSE          &      time (s)          \\ \hline
		nuclear&     0.0436$\pm$0.0003     &     2.1$\pm$0.2
		&   0.0375$\pm$0.0003      &   4.2$\pm$1.0     &      0.0333$\pm$0.0001  &   40.9$\pm$7.2   \\ \hline
		factored nuclear    &     0.0246$\pm$0.0003      &    \textbf{0.04$\pm$0.01}      &    0.0218$\pm$0.0004      &    \textbf{0.08$\pm$0.02}      &     0.0198$\pm$0.0001    &   \textbf{0.4$\pm$0.2}     \\\hline  
		BPMF		&0.0234$\pm$0.0005 &3.2$\pm$0.4&0.0203$\pm$0.0005&5.8$\pm$0.9
		&0.0188$\pm$0.0001&48.3$\pm$5.9\\\hline
		factored GSR&     0.0219$\pm$0.0003     &     0.5$\pm$0.1
		&   0.0197$\pm$0.0004      &   4.2$\pm$0.2     &      0.0185$\pm$0.0001  &  6.7$\pm$0.4   \\ \hline
		truncated $\lonetwo$ &\textbf{0.0196$\pm$0.0003} &695.8$\pm$19.2&\textbf{0.0182$\pm$0.0004}&1083.2$\pm$40.78&\textbf{0.0177$\pm$0.0001} &3954.1$\pm$98.7\\\hline
		capped-$\ell_1$
		& \textbf{0.0197$\pm$0.0003}  &    0.8$\pm$0.1      &     \textbf{0.0183$\pm$0.0003}   &    5.4$\pm$0.1     & \textbf{0.0178$\pm$0.0001}  & 36.0$\pm$3.4      \\ \hline
		LSP     
		& \textbf{0.0197$\pm$0.0003}  &     0.8$\pm$0.1      & \textbf{0.0183$\pm$0.0004}  &   5.1$\pm$0.1      &\textbf{0.0177$\pm$0.0001}  &  35.1$\pm$2.1     \\ \hline
		MCP
		& \textbf{0.0196$\pm$0.0003}  &   0.7$\pm$0.1      & \textbf{0.0182$\pm$0.0003}  &    4.1$\pm$0.2     & \textbf{0.0178$\pm$0.0001}  &    40.6$\pm$3.6    \\ \hline
		NNFN &\textbf{0.0196$\pm$0.0003} &2.1$\pm$0.2&\textbf{0.0182$\pm$0.0003} &7.7$\pm$0.6
		&\textbf{0.0177$\pm$0.0001} &43.1$\pm$2.3\\\hline
		\multirow{1}{*}{factored NNFN}&\textbf{0.0196$\pm$0.0003} &\textbf{0.04$\pm$0.01}&\textbf{0.0182$\pm$0.0003} &\textbf{0.08$\pm$0.02} &\textbf{0.0177$\pm$0.0001}&\textbf{0.3$\pm$0.1}\\\hline
	\end{tabular}
	\label{tab:syn_expts}
\end{table*}
\begin{figure*}[htp]
	\centering
	\subfigure[$m=500$.]{\includegraphics[width =
		0.246\textwidth]{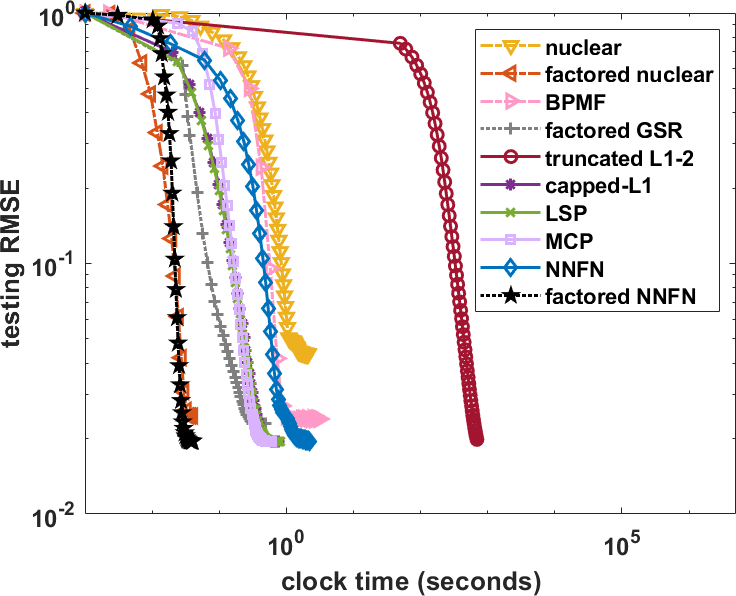}}
	\hspace{40px}
	\subfigure[$m=1000$.]{\includegraphics[width =
		0.246\textwidth]{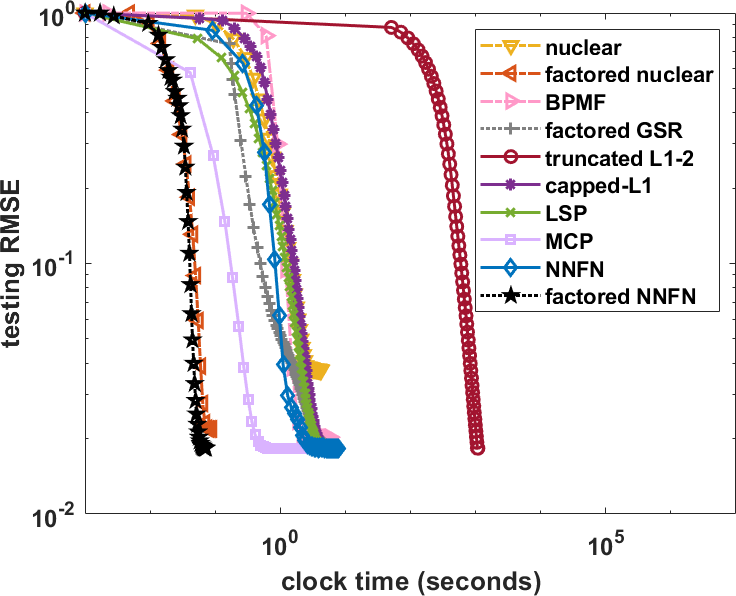}}
	\hspace{40px}
	\subfigure[$m=2000$.\label{fig:syn_2000}]{\includegraphics[width =
		0.246\textwidth]{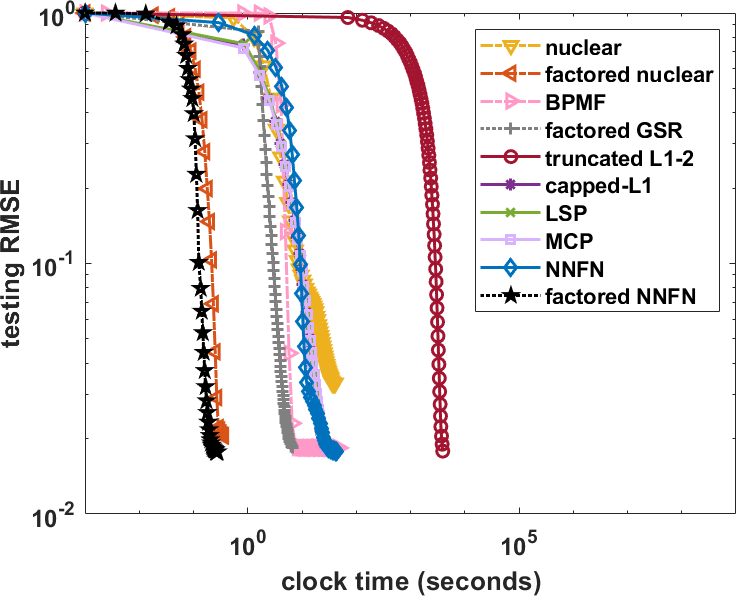}}
	\caption{Testing NMSE versus clock time on the synthetic data sets.
	}
	\label{fig:syn_nmse}
\end{figure*}

Note that learning with factored regularizers solves \eqref{eq:mc_fact}, which reduces to matrix factorization when $\mu=0$. Therefore, we do not additionally compare with matrix factorization methods \cite{srebro2005maximum,wen2012solving,tu2016low}.   

All the algorithms are implemented in MATLAB (with sparse 
operations written in C as MEX functions). 
Each algorithm is stopped when the relative difference
between objective values in consecutive iterations is smaller than 
${10}^{-4}$.  
All hyperparameters including stepsize, $\lambda$, $k$, ${r}_t$ and hyperparameters of baseline methods are tuned by grid search using the validation set. 
Specifically, 
$\lambda$ in (\ref{eq:mc_reg}) is chosen from $[10^{-3},10^2]$, ${r}_t$ and $k$ is a integer chosen from $[1,\min(m,n)]$, and stepsize is chosen from $[10^{-5},1]$.
For the  other baselines, we use 
the 
hyperparameter
ranges
as mentioned in the respective papers.

\subsubsection{Evaluation Metrics}
Given an incomplete matrix $\bO$, 
let 
$\mathbf{\Omega}^{\bot}$ record positions of the unobserved
elements (i.e., $\Omega_{ij}^{\bot}=0$ if $O_{ij}$ is observed, and $1$ otherwise),
and
$\bar{\bX}$ be the matrix recovered.
Following \cite{rao2015collaborative,yao2019large},
performance on the synthetic data
is measured by  the normalized mean squared error (NMSE) on
$\mathbf{\Omega}^{\bot}$:
$\text{NMSE} = \NM{P_{\mathbf{\Omega}^{\bot}}(\bar{\bX} - \mathbf{G})}{F} 
/ \NM{P_{\mathbf{\Omega}^{\bot}}(\mathbf{G})}{F}$, where $\mathbf{G}$ is the
ground-truth matrix.
On the real-world data sets, we use the root mean squared error (RMSE) on 
$\mathbf{\Omega}^\bot$:
$\text{RMSE} = (\NM{\mathcal{P}_{\mathbf{\Omega}^{\bot}}(\bar{\bX} -
	\bO)}{F}^2 / \NM{\mathbf{\Omega}^{\bot}}{0})^{\nicefrac{1}{2}}$. 
Besides the error, we also report the training time in seconds.

\subsection{Synthetic Data}
\label{sec:expts_syn}

First, 
$\bW,\bH \in \R^{m \times k^*}$
are generated 
with elements
sampled i.i.d. from the standard normal distribution $\mathcal{N}(0, 1)$.
We set $k^* = 5$, 
and vary $m$ in $\{500, 1000, 2000\}$.
The 
$m\times m$ 
ground-truth matrix (with
rank $k^*$)
is then constructed as $\mathbf{G} =\bW \bH^\top$.
The observed matrix $\bO$ is generated as $\bO = \mathbf{G} + \mathbf{E}$, where 
the elements of $\mathbf{E}$ are sampled from $\mathcal{N}(0, 0.1)$. 
A set of
$\NM{\mathbf{\Omega}}{0} = 2 m k^* \log(m)$ random elements
in $\bO$ are observed, where 50\% of them are randomly sampled for training, and the rest is taken as validation
set for hyperparameter tuning.
We define the sparsity ratio $s$ of the observed matrix as its fraction of observed elements 
(i.e., $s = \NM{\mathbf{\Omega}}{0}/m^2$).

\subsubsection{Performance}

Table~\ref{tab:syn_expts}
shows
the results.
As can be seen, nonconvex regularizers (including the proposed
$\rnf$)
consistently yield better recovery performance. 
Among the nonconvex regularizers, 
all of them yield comparable errors. 
Additionally, we calculate the rank of recovered matrices and find that all methods (except
the nuclear norm regularizer) can 
recover the true rank. 
As for speed,
factored NNFN allows significantly faster optimization than NNFN, which validates the efficiency of using the factored form. 
Only factored nuclear norm regularizer is comparable to factored NNFN in speed (but it is much worse in terms of recovery performance), 
and both 
are orders of magnitudes faster than the others.
Optimization with the truncated $\lonetwo$ is exceptionally slow, which is due to
the need of having 
two levels of 
DCA
and 
ADMM 
iterations. 
The convergence of testing NMSE is put in Figure~\ref{fig:syn_nmse}, which also shows 
factored NNFN always has 
the fastest 
convergence
to the lowest NMSE.

\begin{figure*}[h!]
	\centering
	{\includegraphics[width =
		0.236\textwidth]{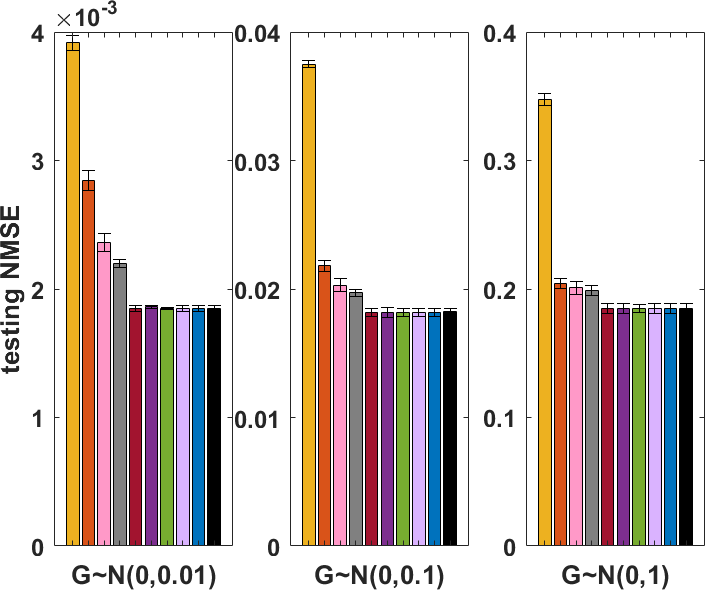}}
	\hspace{40px}
	{\includegraphics[width =
		0.246\textwidth]{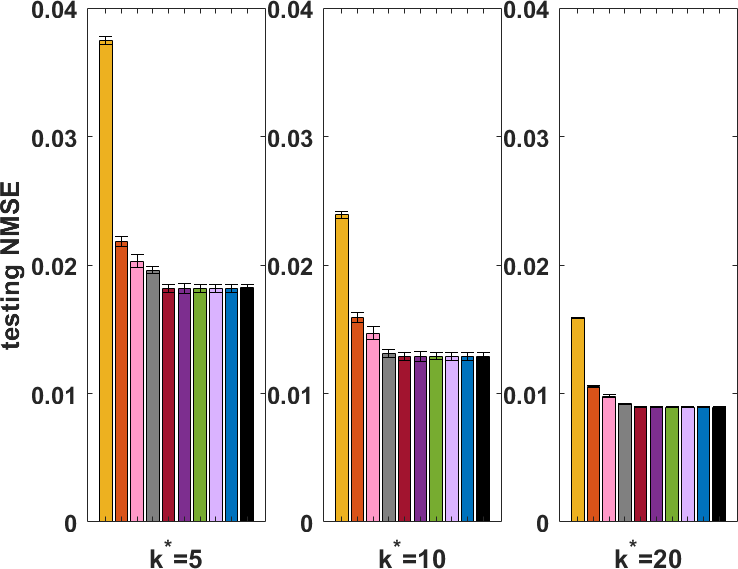}}
	\hspace{40px}
	{\includegraphics[width =
		0.246\textwidth]{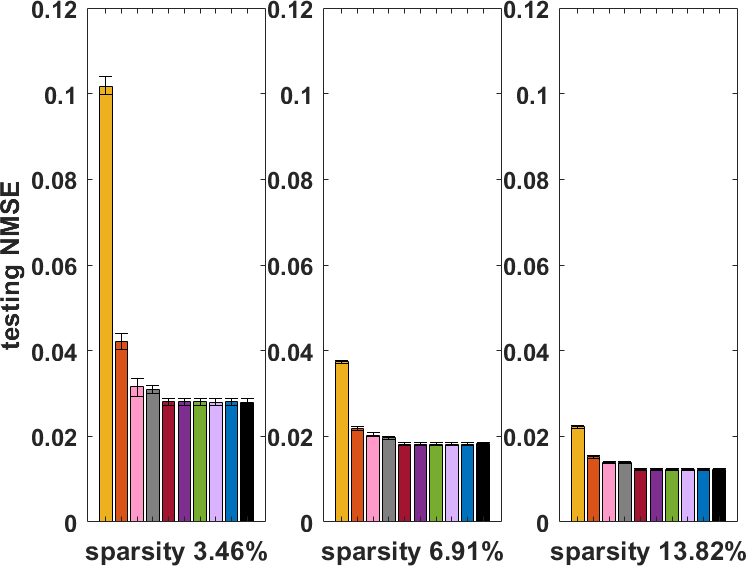}}
	\subfigure[different noise variances.]{\includegraphics[width =
		0.246\textwidth]{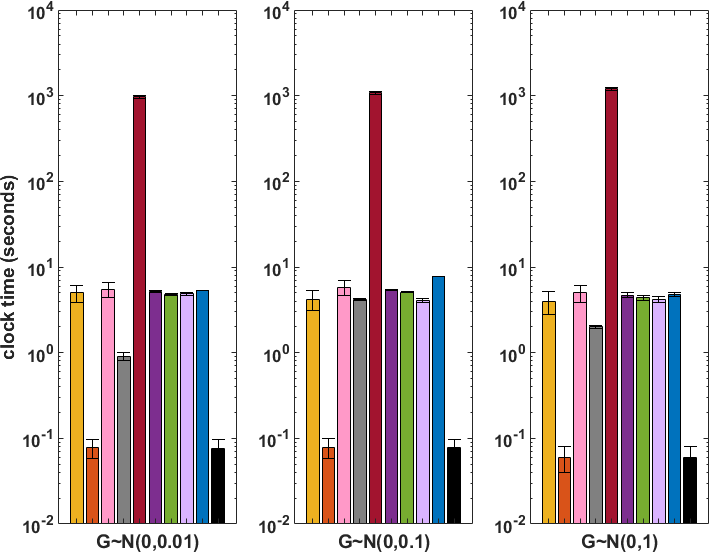}}	
	\hspace{40px}
	\subfigure[different ranks.]{\includegraphics[width =
		0.246\textwidth]{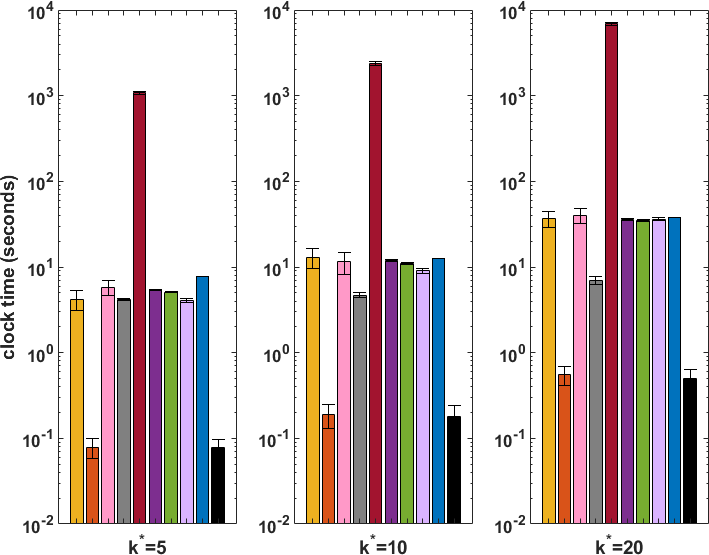}}
	\hspace{40px}
	\subfigure[different sparsity ratios.]{\includegraphics[width =
		0.246\textwidth]{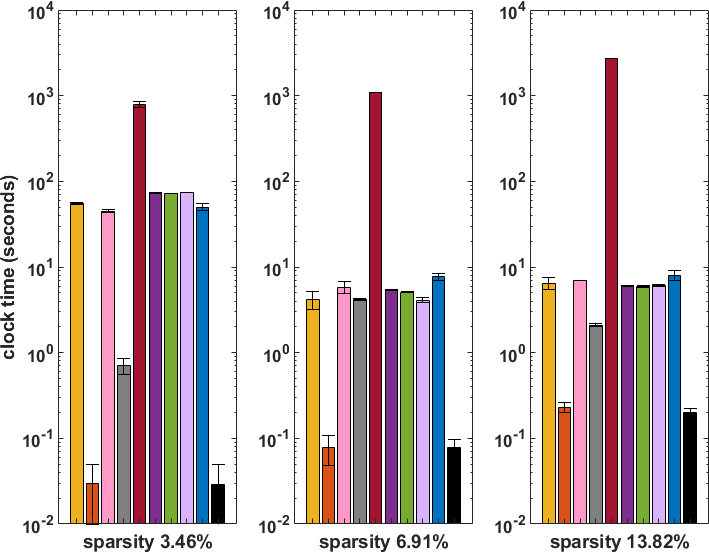}}
	\includegraphics[width=1.8\columnwidth]{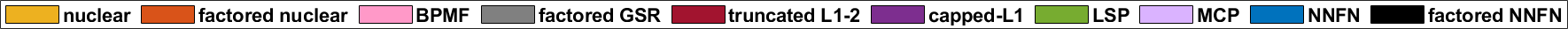}	
	\caption{Testing NMSE (first row) and clock time (second row)  with different settings on the synthetic data  ($m=1000$). 
	The default setting is $k^* = 5$,  $\mathbf{E}\sim\mathcal{N}(0, 0.1)$ and $s=6.91\%$.	For each figure, we only vary one variable while keeping the others as the default setting.}
	\label{fig:ablation_nmse}
\end{figure*}

\subsubsection{Effects of Noise, Rank and Sparsity Ratio} 
In this section, we vary (i) the 
variance of the Gaussian noise matrix $\mathbf{E}$ in the range $\{0.01, 0.1,1\}$;
(ii) the
true rank $k^*$ 
of the data in $\{5,10,20\}$; and (iii)
the sparsity ratio $s$ in $\{0.5,1,2\} \times (2mk^*\log(m)/m^2)$.  
The experiment is performed
	on the synthetic data set, with $m=1000$. 
In each trial, we 
only vary one variable while keeping the others at default, 
i.e., $k^* = 5$,  $\mathbf{E}\sim\mathcal{N}(0, 0.1)$ and $s=6.91\%$.	
Figure~\ref{fig:ablation_nmse} shows
the testing NMSE results and the timing results. 
As expected, a larger noise, 
smaller true rank, or sparser matrix lead to a
harder matrix completion problem and subsequently higher NMSE's.
However, the relative performance ranking of the various methods remain the same,
and nonconvex regularization always obtain a smaller NMSE. 
For time, 
although the exact timing results vary across different settings, consistent observation can be made: 
factored NNFN is consistently faster than the others. 

\begin{table*}[htbp]
	\centering
	\caption{Performance on the recommendation  data sets. 
		Entries marked as ``-" mean that the corresponding 
		methods cannot complete in three hours.
		The best and comparable results (according to the pairwise t-test with 95\% confidence) are highlighted in bold.
	}
	\begin{tabular}{c|cc|cc|cc}
		\hline
		&       \multicolumn{2}{c|}{\textit{MovieLens-1M}} &       \multicolumn{2}{c|}{\textit{MovieLens-10M}}& \multicolumn{2}{c}{\textit{Yahoo}}                \\
		& testing RMSE &       time (s)       & testing RMSE &       time (s)  & testing RMSE &       time (s) 
		\\ \hline
		\multirow{1}{*}{nuclear}     &     0.820$\pm$0.002      & 118.7$\pm$19.2& 0.807$\pm$0.001&821.2$\pm$27.7 &   0.721$\pm$0.001 &     1133.1$\pm$58.3     \\\hline
		\multirow{1}{*}{factored nuclear}         &    0.810$\pm$0.001&   0.6$\pm$0.1&0.795$\pm$0.001 & \textbf{41.8$\pm$7.3}&0.710$\pm$0.008&     533.9$\pm$25.7\\\hline
		\multirow{1}{*}{BPMF}       &0.807$\pm$0.001&215.3$\pm$29.4&0.791$\pm$0.001 & 819.6$\pm$30.8&0.707$\pm$0.003&1433.5$\pm$89.2\\
		\hline
		\multirow{1}{*}{factored GSR}	&0.805$\pm$0.001&14.2$\pm$1.5& -&- &-&-\\	 \hline
		\multirow{1}{*}{truncated $\ell_{1-2}$}	&\textbf{0.797}\textbf{$\pm$0.001}&6068.4$\pm$172.0& -&- &-&-\\\hline  
		\multirow{1}{*}{capped-$\ell_1$} & 0.800$\pm$0.001&     147.9$\pm$23.3&0.787$\pm$0.001 &812.3$\pm$29.7 &0.658$\pm$0.001& 1296.8$\pm$67.3\\ \hline
		\multirow{1}{*}{LSP}      & 0.799$\pm$0.001&     149.2$\pm$23.5&0.787$\pm$0.001 &850.8$\pm$31.1 & 0.656$\pm$0.001&     1078.0$\pm$69.0 \\\hline
		\multirow{1}{*}{MCP}      & 0.801$\pm$0.001&     151.4$\pm$23.9&0.787$\pm$0.001 &849.8$\pm$31.5 &  0.678$\pm$0.001&    1108.3$\pm$41.4 \\ \hline   
		\multirow{1}{*}{NNFN}        &    \textbf{0.797}\textbf{$\pm$0.001}     & 134.2$\pm$19.6&\textbf{0.782$\pm$0.001} &834.5$\pm$29.2 &   \textbf{0.652$\pm$0.001} &     1209.7$\pm$61.2 \\\hline
		\multirow{1}{*}{factored NNFN}     &\textbf{0.797}\textbf{$\pm$0.001}&\textbf{0.5}\textbf{$\pm$0.1}& \textbf{0.782}\textbf{$\pm$0.001} &\textbf{40.0}\textbf{$\pm$5.5} &\textbf{0.652}\textbf{$\pm$0.001}&\textbf{522.5}\textbf{$\pm$21.9}\\ 
		\hline
	\end{tabular}
	\label{tab:rec_perf}
\end{table*}
\begin{figure*}[h!]
	\centering
	\subfigure[\textit{MovieLens-1M}.\label{fig:large_movielens1m}]{\includegraphics[width =
		0.250\textwidth]{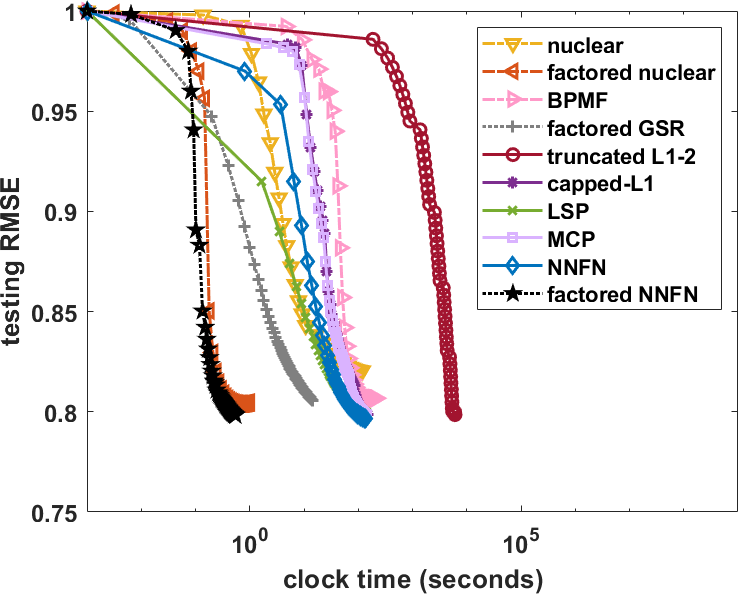}}
	\hspace{40px}
	\subfigure[\textit{MovieLens-10M}.\label{fig:large_movielens10m}]{\includegraphics[width =
		0.245\textwidth]{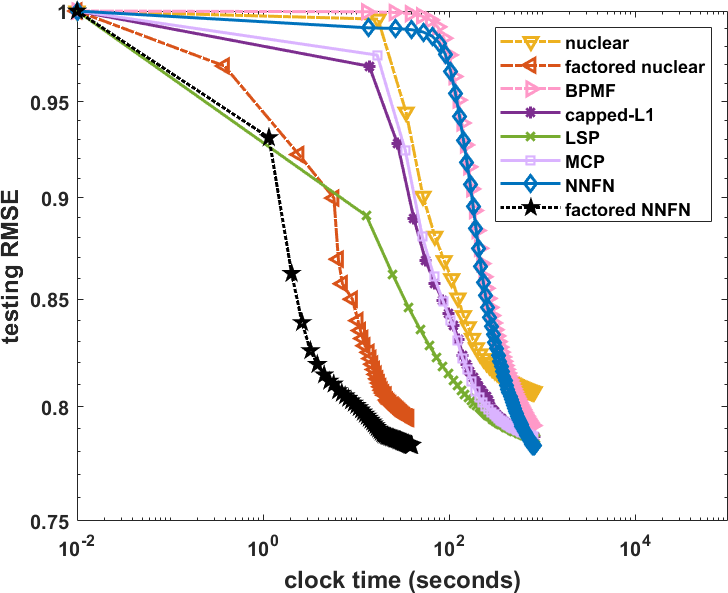}}
	\hspace{40px}
	\subfigure[\textit{Yahoo}.\label{fig:large_yahoo}]
	{\includegraphics[width = 0.245\textwidth]{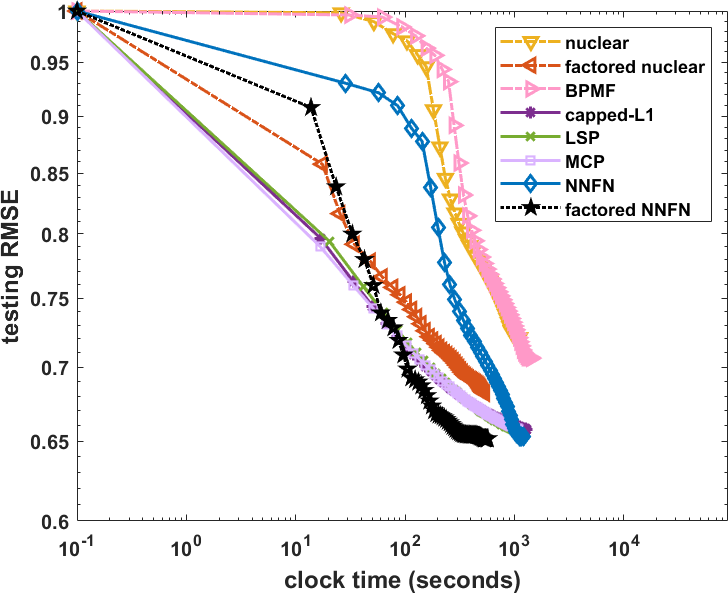}}	
	\caption{Testing RMSE versus clock time on recommendation 
		data.}
	\label{fig:rec_converge}
\end{figure*}

\begin{table*}[htbp]
	\centering
	\caption{Performance on the climate data sets. 
		The best and comparable results (according to the pairwise t-test with 95\% confidence) are highlighted in bold.
	}
	\begin{tabular}{c|cc|cc|cc|cc}
		\hline
		&       \multicolumn{4}{c|}{\textit{GAS}}     & \multicolumn{4}{c}{\textit{USHCN}}    \\
		\hline
		&       \multicolumn{2}{c|}{$\textit{CO}_{2}$}     & \multicolumn{2}{c|}{$\textit{H}_{2}$} &       \multicolumn{2}{c|}{\textit{temperature}}     & \multicolumn{2}{c}{\textit{precipitation}}    \\
		&          testing RMSE           &         time (s)        &          testing RMSE
		&         time (s)   &          testing RMSE          &         time (s)        &
		testing RMSE           &         time (s)        \\ \hline
		\multirow{1}{*}{nuclear}  &   0.584$\pm$0.005     &  1.0$\pm$0.1  &0.593$\pm$0.006 &0.8$\pm$0.1 &0.480$\pm$0.014     &  108.5$\pm$4.1   &0.828$\pm$0.020 &84.9$\pm$11.2\\ 
		\hline
		\multirow{1}{*}{factored nuclear}    &     0.565$\pm$0.006      & \textbf{0.05$\pm$0.02}     & 0.574$\pm$0.005&\textbf{0.06$\pm$0.03} & 0.483$\pm$0.016      & \textbf{6.0$\pm$1.6}     & 0.823$\pm$0.018&\textbf{13.7$\pm$1.7}  \\  \hline
		\multirow{1}{*}{BPMF}&0.552$\pm$0.005&3.2$\pm$0.3&0.554$\pm$0.005&3.4$\pm$0.4&0.464$\pm$0.012&148.1$\pm$7.7&0.819$\pm$0.015&125.1$\pm$8.1\\\hline
		\multirow{1}{*}{GRALS}&0.565$\pm$0.006&0.7$\pm$0.1&0.578$\pm$0.005&0.4$\pm$0.1&0.498$\pm$0.015&37.2$\pm$1.9&0.818$\pm$0.016&49.6$\pm$2.8\\\hline
		\multirow{1}{*}{truncated $\ell_{1-2}$}	&\textbf{0.530}\textbf{$\pm$0.007}&11.0$\pm$1.2&\textbf{0.531}\textbf{$\pm$0.005}&7.3$\pm$1.8&\textbf{0.444}\textbf{$\pm$0.013}&573.9$\pm$18.1&\textbf{0.806}\textbf{$\pm$0.014}&318.5$\pm$9.7\\	\hline 
		\multirow{1}{*}{capped-$\ell_1$} 
		& 0.533$\pm$0.003 &    0.6$\pm$0.1      & \textbf{0.531}\textbf{$\pm$0.005} &0.7$\pm$0.2 &0.450$\pm$0.014 &    108.5$\pm$10.5      &\textbf{0.806}\textbf{$\pm$0.014}&87.2$\pm$6.2  \\ \hline
		\multirow{1}{*}{LSP}      
		& 0.537$\pm$0.008  &     1.2$\pm$0.1    & 0.540$\pm$0.007& 1.3$\pm$0.2&0.448$\pm$0.010  &     133.3$\pm$7.7     & \textbf{0.806}\textbf{$\pm$0.014}& 105.8$\pm$7.4\\ \hline
		\multirow{1}{*}{MCP}       
		& \textbf{0.530}\textbf{$\pm$0.008}  &    1.0$\pm$0.1	     & 0.534$\pm$0.006& 0.5$\pm$0.1& \textbf{0.444}\textbf{$\pm$0.013}  &    92.5$\pm$6.1	     &\textbf{0.806}\textbf{$\pm$0.014}& 85.2$\pm$7.3\\  \hline    
		\multirow{1}{*}{NNFN} &\textbf{0.530}\textbf{$\pm$0.008} &0.4$\pm$0.1		&  \textbf{0.531}\textbf{$\pm$0.005}&0.5$\pm$0.1&\textbf{0.444}\textbf{$\pm$0.012} &57.1$\pm$3.3		&  \textbf{0.806}\textbf{$\pm$0.014}&66.4$\pm$5.1 \\\hline
		\multirow{1}{*}{factored NNFN}&\textbf{0.530}\textbf{$\pm$0.006}&\textbf{0.05}\textbf{$\pm$0.01}& \textbf{0.531}\textbf{$\pm$0.005} &\textbf{0.05}\textbf{$\pm$0.02}&\textbf{0.444}\textbf{$\pm$0.012}&\textbf{5.9}\textbf{$\pm$1.4}& \textbf{0.806}\textbf{$\pm$0.015} &\textbf{13.5}\textbf{$\pm$1.9}\\\hline
	\end{tabular}
	\label{tab:climate_perf}
\end{table*}
\begin{figure*}[htbp]
	\centering
	\subfigure[\textit{GAS}-$\textit{CO}_{2}$.]{\includegraphics[width =
		0.246\textwidth]{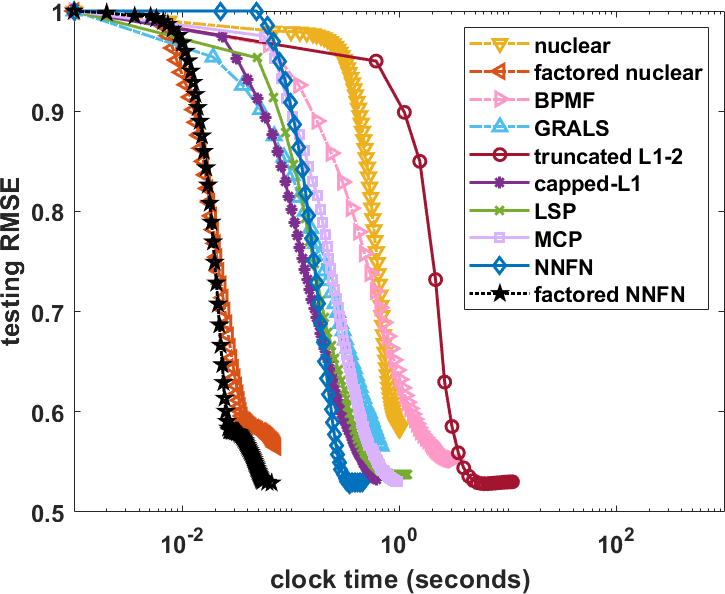}}
	\subfigure[\textit{GAS}-$\textit{H}_{2}$.]{\includegraphics[width =
		0.246\textwidth]{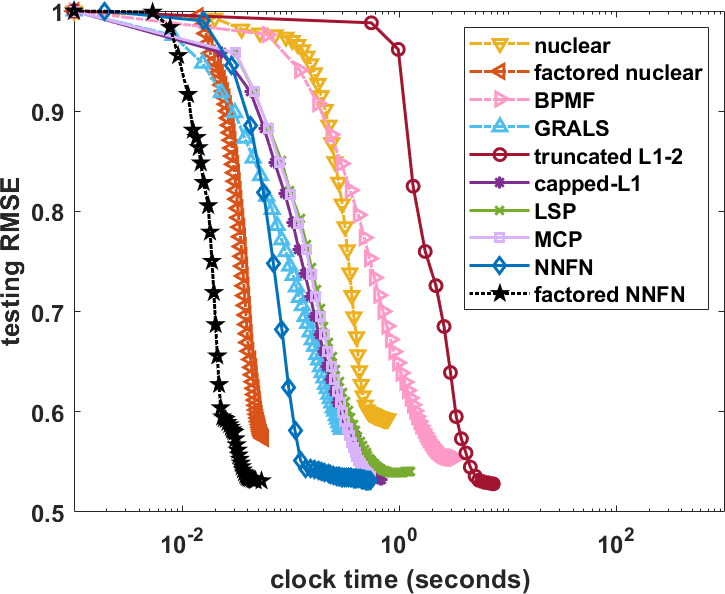}}
	\subfigure[\textit{USHCN}-\textit{temperature}.]{\includegraphics[width =
		0.246\textwidth]{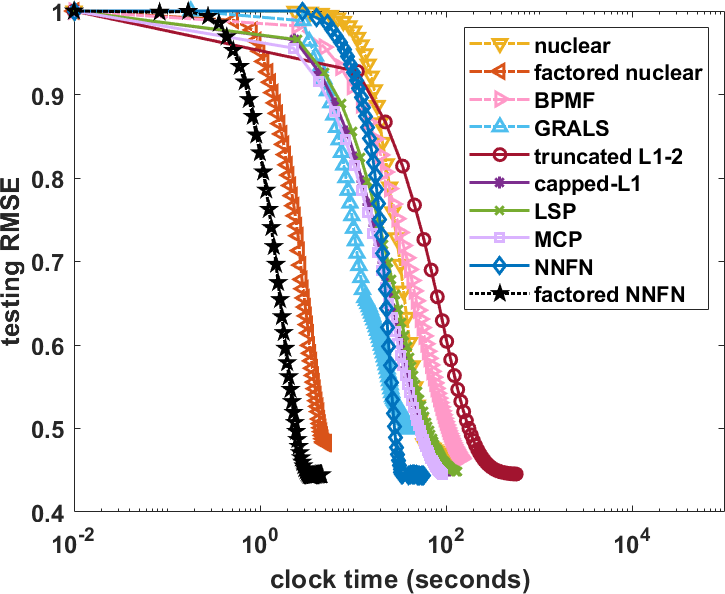}}
	\subfigure[\textit{USHCN}-\textit{precipitation}.]{\includegraphics[width =
		0.246\textwidth]{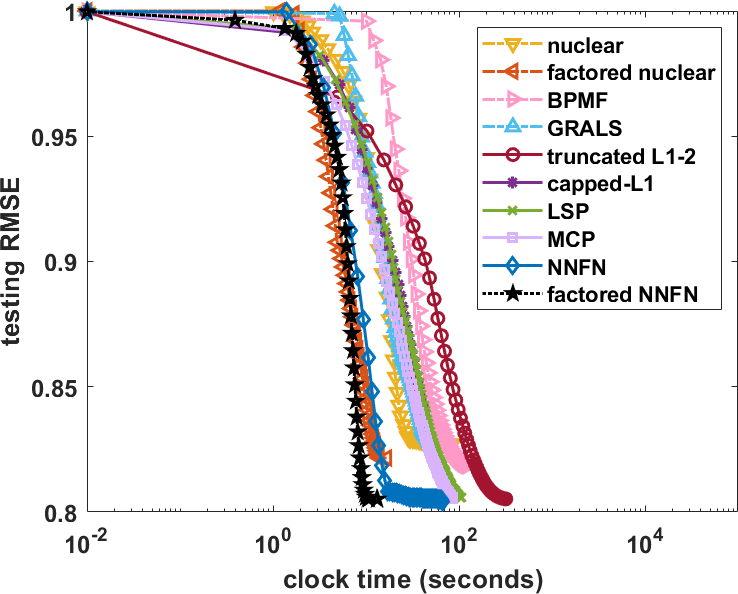}}	
	\caption{Testing RMSE versus clock time on climate data.}
	\label{fig:climate_converge}
\end{figure*}

\subsection{Recommendation Data} 
In this section, experiments are performed on the 
popularly used benchmark recommendation data sets: 
\textit{MovieLens-1M} data set \cite{harper2015movielens} (of size $6,040\times 3,449$),  
\textit{MovieLens-10M} data set\footnote{\url{http://grouplens.org/datasets/movielens/}} \cite{harper2015movielens} (of size $69,878\times10,677$) 
and \textit{Yahoo}\footnote{\url{http://webscope.sandbox.yahoo.com/catalog.php?datatype=c}} \cite{koren2009matrix} data set (of size $249,012\times 296,111$).
We uniformly sample 50\% of the ratings as observed for training, 25\% for
validation (hyperparameter tuning) and the rest for testing. 
Optimization with the truncated $\lonetwo$ cannot converge in three hours on \textit{MovieLens-10M} and \textit{Yahoo}, while 
Factored GSR runs out of memory on \textit{MovieLens-10M} and \textit{Yahoo} as it requires full matrices.  
Thus, their results are not reported.

Table~\ref{tab:rec_perf} shows the results. 
As can be seen, nonconvex regularizers obtain the best recovery performance than the other methods. Among them, factored NNFN is again the fastest.
Figure~\ref{fig:rec_converge} shows
convergence of the testing RMSE.
Consistent observation can be made, factored NNFN always obtains the best performance while being the fastest.

\subsection{Climate Data}  
Additionally, we evaluated the proposed method on climate record data sets.
The \textit{GAS}\footnote{\url{https://viterbi-web.usc.edu/~liu32/data/NA-1990-2002-Monthly.csv}} 
and \textit{USHCN}\footnote{\url{http://www.ncdc.noaa.gov/oa/climate/research/ushcn}} data sets from
\cite{bahadori2014fast} are used. 
\textit{GAS} contains monthly observations for the green gas
components from January 1990 to December 2001, of which we use $\textit{CO}_{2}$ and $\textit{H}_{2}$.
\textit{USHCN} contains monthly temperature
and precipitation readings from January 1919 to November 2019.
For these two data sets, 
some rows (which correspond to locations) of the observed matrix are completely missing. 
The task is to predict 
climate observations for locations that do not have any records.
Following \cite{bahadori2014fast}, we normalize the data to zero mean and unit variance, 
then randomly sample 10\% of the locations for training, 
another 10\% for validation,
and the rest
for testing.
To allow generalization to these completely unknown locations,
we follow \cite{bahadori2014fast} and
add a graph Laplacian regularizer to
\eqref{eq:mc_reg}. 
Specifically, the $m$ locations are represented as nodes on a graph. 
The affinity matrix $\bA=[A_{ij}]\in\R^{m\times m}$, which contains pairwise node
similarities, is computed as
$A_{ij}=\exp(-2b(i,j))$, where $b(i,j)$ is the Haversine
distance between locations $i$ and $j$.
The graph Laplacian regularizer is then defined as $a(\bX)= \Tr{\bX^\top(\bD-\bA)
	\bX}$,
where $D_{ii}=\sum_j A_{ij}$.
For the factored models (factored nuclear norm, BPMF and factored NNFN), we write $a(\bX)$ as $a(\bW\bH^\top)$. 
Additionally, we compare with graph regularized alternating least squares (GRALS) \cite{rao2015collaborative}, which optimizes for factored nuclear norm with $a(\bW)$.  

Results are shown in Table~\ref{tab:climate_perf}. 
They are consistent with the observations in the previous experiments.
In terms of recovery performance, 
all the nonconvex regularizers 
(including NNFN and factored NNFN)
have comparable performance and 
obtain the lowest testing RMSE.
In terms of speed, factored NNFN and factored nuclear norm are again the fastest, and this  
speed advantage is 
particularly apparent on the larger \textit{USHCN} data set. 
Figure~\ref{fig:climate_converge} shows the convergence of RMSE. 
As shown, nonconvex regularizers generally obtain better testing RMSEs.
Among them, 
factored NNFN is the fastest in convergence.
This again validates the
efficiency and effectiveness of factored NNFN.


\section{Conclusion}

We propose a scalable, adaptive and sound nonconvex regularizer for low-rank matrix learning. 
This regularizer can adaptively penalize singular values as common nonconvex regularizers. 
Further, 
we discover that learning with its factored form can be 
optimized by general solvers such as gradient-based method. We provide theoretical analysis for  recovery and convergence guarantee. 
Extensive experiments on matrix completion problem 
show that
the proposed algorithm achieves state-of-the-art recovery performance, 
while being the fastest among existing low-rank convex / nonconvex regularization and factored regularization methods. 
In sum, the proposed method can be useful to solve many large-scale matrix learning problems in the real world.

{
	\bibliographystyle{ACM-Reference-Format}
	\bibliography{low-rank-full-name}
}

\appendix

\section{Proof}
\label{app:proof}

\subsection{Proposition~\ref{pr:property}}
\label{app:pr_property}

\begin{lemma}\label{lem:app1} 
	When 
	Let $r(\bX)$ be a nonconvex low-rank regularizer of the form \eqref{eq:noncvx}, and 
	$\hat{r}(\alpha)$ is defined as in Table~5.  
	As analyzed in \cite{lu2015generalized,yao2019large}, $\hat{r}$ is a nonlinear, 
	concave and non-decreasing function for $\alpha \ge 0$
	with $\hat{r}(0) = 0$, 
	and
	\begin{align}\label{eq:pz}
	p = \arg\min_{x} \frac{1}{2} (x - z)^2 + \lambda \hat{r}(| x |).
	\end{align}
	Then,
	when $z \ge 0$, 
	we have
	\begin{itemize}
		\item $0 \le p \le z$,
		
		\item $z_1 - p(z_1) \le z_2 - p_2$ for $z_1 \ge z_2$.
	\end{itemize}
\end{lemma}

\begin{proof}
	The first point:
	Obvious, when $z \ge 0$, $p \ge 0$. 
	Next, we prove $p \le z$ by  contradiction. 
	Assume that 
	$p>z$. 
	Then as $\hat{r}$ is non-decreasing, 
	we have 
	$\hat{r}(| p |)\ge \hat{r}(| z |)$.
	Therefore we get 
	\begin{align*}
	\frac{1}{2} (p - z)^2 + \lambda \hat{r}(|z'|)
	>
	\frac{1}{2} (z - z)^2 + \lambda \hat{r}(|z|),
	\end{align*}
	which leads to a contradiction that $p$ is the minimum solution found by optimizing \eqref{eq:pz}.
	Thus,
	$p \le z$.
	
	\noindent
	The second point:
	The optimal of \eqref{eq:pz} $p$ satisfies 
	\begin{align}\label{eq:tmp1}
	p -z+ \lambda \partial \hat{r}( p )=0.
	\end{align}
	As $\hat{r}$ is concave,  for $z_1\ge z_2$, 
	we have 
	\begin{align}\label{eq:tmp2}
	\partial \hat{r}( p_1 ) \le \partial \hat{r}( p_2 ).
	\end{align}
	Combining \eqref{eq:tmp2} and \eqref{eq:tmp1}, we get 
$z_1 - p_1 
	\le 
	z_2 - p_2$,
	and the second point is proved.
\end{proof}

Now,
we can prove Proposition~\ref{pr:property}.

\begin{proof}	
	Let the SVD of $\bZ$ be $\bar{\bU}\Diag{{\bsig}(\bZ)}\bar{\bV}^\top$. 
	From Theorem~1 in \cite{lu2015generalized},
	we have
$\bar{\bX}\equiv \Prox{\lambda r}{\bZ}
	= \bar{\bU} \Diag{\tilsig} \bar{\bV}^{\top}$, 
	where $\tilde{\sigma}_i = \Prox{\lambda \hat{r}}{{\sigma}_i(\bZ)}$.
	Then, as every singular value $\sigma_{i}(\bZ)$ is nonnegative, 
	by Lemma~\ref{lem:app1},
	we get the conclusion.
	Note that
	since $\hat{r}$ is not linear,
	thus the strict inequality 
	in $\sigma_i(\bZ) - \tilde{\sigma}_i \le{\sigma}_{i + 1}(\bZ) -\tilde{\sigma}_{i+1}$
	holds at least for one $i$.
\end{proof}

\subsection{Proposition~\ref{pr:l12_to_nnfn}}
\label{app:l12_to_nnfn}

\begin{proof}
	Let $\bX=\bU\Diag{\tilsig}\bV^\top$ and $\bZ = \bar{\bU}\Diag{\bsig(\bZ)}\bar{\bV}^\top$ be the SVD decomposition of $\bX$ and $\bZ$. 
	By simple expansion, 
	we have
	\begin{align*}
	&\frac{1}{2}\NM{\bX - \bZ}{F}^2 + \lambda \rnf(\bX)\\
	&= \frac{1}{2}\Tr{\bX^\top\bX+\bZ^\top\bZ-2\bX^\top\bZ}+\lambda(\NM{\bX}{*}-\theta\NM{\bX}{F})\\
	& =  \frac{1}{2}(\NM{\tilsig}{2}^2+\NM{\bsig(\bZ)}{2}^2)-\Tr{\bX^\top\bZ}+\lambda \NM{\tilsig}{1}-\lambda\theta\NM{\tilsig}{2}.
	\end{align*}
	Recall that 
	$\Tr{\bX^\top\bZ}\le\tilsig^\top\bsig(\bZ)$
	achieves its equality at $\bU=\bar{\bU}, \bV=\bar{\bV}$  \cite{jennings1992matrix}. 
	Then solving \eqref{eq:prox_nnfn} can be instead computed by solving $\tilsig$ as  
	\begin{align}\label{eq:prox_nnfn_12}
	\arg\min_{\tilsig}& \frac{1}{2}\NM{\tilsig-\bsig(\bZ)}{2}^2 + \lambda(\NM{\tilsig}{1}-\theta\NM{\tilsig}{2})\\\notag
	\st& \tilde{\sigma}_1\ge \tilde{\sigma}_2\ge\dots\ge \tilde{\sigma}_m \ge 0.
	\end{align}
	It can be solved by proximal operator as
	$\tilsig= \Prox{\lambda \NM{\cdot}{1\text{-}2}}{\bsig(\bZ)}$. 
	The constraint in \eqref{eq:prox_nnfn_12} is naturally satisfied. 
	As $\sigma_{i}(\bZ)\ge0$ and $\sigma_{i}(\bZ)\ge \sigma_{i+1}(\bZ)$ $\forall i$, we must have 
	$\tilde{\sigma}_i\ge0$ and $\tilde{\sigma}_{i}\ge \tilde{\sigma}_{i+1}$ $\forall i$. Otherwise, we can always swap the sign or value of $\tilde{\sigma}_i$ and $\tilde{\sigma}_{i+1}$ and obtain a smaller objective of \eqref{eq:prox_nnfn_12}. 
\end{proof}

\subsection{Corollary~\ref{cr:property_other}}

\begin{proof}
	
	Let $\bX=\bU\Diag{\tilsig}\bV^\top$ where $\tilsig= \Prox{\lambda \NM{\cdot}{1\text{-}2}}{\bsig(\bZ)}$. 
	Now we prove that 
	\begin{itemize}
		\item shrinkage: 
		${\sigma}_i(\bZ) \ge \tilde{\sigma}_i$, 
		\item  adaptivity: 
		$\sigma_i(\bZ) - \tilde{\sigma}_i \le{\sigma}_{i + 1}(\bZ) -\tilde{\sigma}_{i+1}$,
		where the strict inequality 
		holds at least for one $i$.
	\end{itemize}
	
	\noindent
	The first point: 
	From Proposition~\ref{pr:l12_to_nnfn},
	we can see the optimization problem on matrix \eqref{eq:prox_nnfn} can be transformed to an  optimization problem on singular values \eqref{eq:prox_nnfn_12}. 
	As shown in Proposition~\ref{pr:l12_to_nnfn}, for $\tilsig= \Prox{\lambda \NM{\cdot}{1\text{-}2}}{\bsig(\bZ)}$, we have 
	$\sigma_i(\bZ)\ge \tilde{\sigma}_i\ge 0$. 
	
	\noindent
	The second point: 
	The optimal of \eqref{eq:prox_nnfn_12} satisfies 
	\begin{align*}
	\tilsig- \bsig(\bZ)+\lambda-\lambda\frac{\tilsig}{\NM{\tilsig}{2}}=0. 
	\end{align*} 
	As $\sigma_i(\bZ)\ge \tilde{\sigma}_i\ge 0$, we have 
	\begin{align*}
	\sigma_i(\bZ)-\tilde{\sigma}_i = \lambda -\lambda\frac{\tilde{\sigma}_i}{\NM{\tilsig}{2}}\ge 0.
	\end{align*}
	Then as $\tilde{\sigma}_i\ge \tilde{\sigma}_{i+1}$, 
	we have 
	\begin{align*}
	\lambda-\lambda\frac{\tilde{\sigma}_i}{\NM{\tilsig}{2}}
	\le 
	\lambda-\lambda\frac{\tilde{\sigma}_{i+1}}{\NM{\tilsig}{2}}, 
	\end{align*}
	and correspondingly 
	$\sigma_i(\bZ) - \tilde{\sigma}_i \le{\sigma}_{i + 1}(\bZ) -\tilde{\sigma}_{i+1}$.
	The inequality holds only when $\sigma_i(\bZ) \ne \sigma_{i+1}(\bZ)$.
\end{proof}

\subsection{Theorem~\ref{thm:recovery_error}}

\begin{proof} 	
	The regularized  and constrained low-rank matrix completion problem obtain equivalent solutions \cite{boyd2004convex}. 
	Here, we prove for the constrained problem, but the conclusion applies for both forms. 
	
	Assume sequence $\{\bX^t\}$ with $f(\bX^{t+1})< f(\bX^{t})$ and each $\bX^t$ is the iterate obtained by optimizing the following two 
	equivalent constrained formulations 
	of \eqref{eq:fnnfn}: 
	(i) $\bX^t$ is the iterate of optimizing $\min\nolimits_{\bX} 
	f(\bX)
	\!\;\!\st\!\;\!
	\rnf(\bX) \!\le\! \beta'$, where $\beta'\geq 0$ is a hyperparameter. 
	or (ii)  $\bX^t\!=\!\bW^t(\bH^t)^\top$ is the iterate of optimizing $\min\nolimits_{\bW,\bH} 
	f(\bW\bH^\top)$ $
	\!\;\!\st\!\;\! 
	\frac{1}{2}(\NM{\bW}{F}^2\!+\!\NM{\bH}{F}^2) \!-\! \NM{ \bW\bH^\top }{F} \!\le\! \beta'$, where $\beta'\geq 0$ is a hyperparameter. 
	These sequences can be obtained by optimizing the two constrained problems via projected gradient descent which guarantees sufficient decrease in $f$ \cite{boyd2004convex,nocedal2006numerical}.

	Obviously, 
	the optimal $\bX^*$ satisfies $\bb-\aff(\bX^*)=\be$ and hence $f(\bX^*)=\frac{1}{2}\|\aff(\bX^*)-\bb\|_2^2=\frac{\|\be\|_2^2}{2}$. 
	Thus 
	$\bX^t$ obtained at the $t$th iteration satisfies $f(\bX^t)\geq \frac{c_1^2\|\be\|_2^2}{2}\geq \frac{\|\be\|_2^2}{2}$ for constant $c_1$ whose absolute value is larger than 1.
	By choosing the dimension $k$ of $\bW\in\R^{m\times k}$, one can let $\bX^t\le k^*$, where $k^*$ is the true rank of the optimal matrix $\bX^*$.
	
	We can derive 
	\begin{align}\notag
	\|\aff(\bX^* -\bX^t)\|_2^2 &\leq \|(\bb-\aff(\bX^t))-\be\|_2^2,\\\notag
	&\leq 2\left(f(\bX^t)-\be^\top(\bb-\aff(\bX^t))+\frac{\|\be\|^2}{2}\right),\\\notag
	&\leq 2\left(f(\bX^t)+\frac{2}{c_1}f(\bX^t)+\frac{1}{c_1^2}f(\bX^t)\right),\\\label{eq:bound_x}
	&\leq 2\left(1+\frac{1}{c_1}\right)^2f(\bX^t)
	\end{align}
	
	Now, we are ready to bound the difference between this $\bX^t$ and the optimal $\bX^*$.  
	\begin{align}\label{eq:dif_x_0}
	\NM{\bX^t-\bX^*}{F}^2&\leq \frac{1}{1-\delta_{2k^*}} \|\aff(\bX^t -\bX^*)\|_2^2\\\label{eq:dif_x_1}
	&\leq \frac{2}{1-\delta_{2k^*}}\left(1+\frac{1}{c_1}\right)^2f(\bX^t)\\\label{eq:dif_x_2}
	&\leq \frac{1}{1-\delta_{2k^*}}\left(1+\frac{1}{c_1}\right)^2(c_1^2+\epsilon)\|\be\|^2_2\\\notag
	& = 
	\frac{(c_1+1)^2(c_1^2+\epsilon)\|\be\|^2_2}{c_1^2(1-\delta_{2k^*})},
	\end{align}
	where the isometry constant is $\delta_{2k^*}$ 
	as $\bX^t-\bX^*$ is a matrix of rank at most $2k^*$, 
	\eqref{eq:dif_x_0} is derived from RIP,  	
	\eqref{eq:dif_x_1} comes from \eqref{eq:bound_x},  	
	and 
	\eqref{eq:dif_x_2} is obtained as one can choose a small constant $\epsilon$ such that $\frac{(c_1^2+\epsilon)\|\be\|_2^2}{2}\geq f(\bX^t)\geq \frac{c_1^2\|\be\|_2^2}{2}$.
	
\end{proof}

\subsection{Theorem~\ref{pr:convergence}}

\begin{proof}
	For smooth functions, gradient descent can obtain sufficient decrease as shown in the following Proposition.
	\begin{prop}[\cite{nocedal2006numerical}]\label{pr:app}
		A differentiable function $h$ with $L$-Lipschitz continuous
		gradient, i.e., 
		$\NM{\nabla_{x}h(x^t)-\nabla_{x}h(x^{t+1})}{2}\le L\NM{x^t-x^{t+1}}{2}$,
		satisfies the following inequality, 
		\begin{align*}
		h(x^{t})
		- h(x^{t + 1})\ge \frac{1}{2L}
		\NM{\nabla_{x} h(x^{t})}{F}^2.
		\end{align*}
		Moreover, when $h$ is bounded from below, i.e., $\inf h(x) > -\infty$ and $\lim_{\NM{x}{2}\rightarrow\infty}h(x)=\infty$, optimizing $h$ by gradient descent is guaranteed to converge.
	\end{prop}
	Since $\bW_t \bH_t^{\top} \not= \mathbf{0}$, $F(\bW,\bH)$ is smooth.
	As gradient descent is used, we then have
	\begin{align*}
	&F(\bW^{t}, \bH^t)
	- F(\bW^{t + 1}, \bH^{t + 1})\\
	&\ge \frac{\eta}{2}
	\NM{\nabla_{\bW} F(\bW^{t}, \bH^t)}{F}^2
	+ \frac{\eta}{2} \NM{\nabla_{\bH} F(\bW^{t}, \bH^t)}{F}^2.
	\end{align*}
	At the $(T+1)$th iteration, the difference between $F(\bW^{1}, \bH^1)$ and $F(\bW^{T+1}, \bH^{T+1})$ is calculated as 
	\begin{align}\notag
	& F(\bW^{1}, \bH^1)
	\!-\! F(\bW^{T + 1}, \bH^{T + 1})\\\label{eq:sum_seq}
	 &\ge 
	\sum_{t=1}^{T}
	\frac{\eta}{2}
	\NM{\nabla_{\bW} F(\bW^{t}, \bH^t)}{F}^2
	\!+\! \frac{\eta}{2} \NM{\nabla_{\bH} F(\bW^{t}, \bH^t)}{F}^2.\!
	\end{align}
	As assumed, $
	\lim\nolimits_{
		\NM{\bW}{F} \rightarrow \infty}F(\bW, \cdot) 
	= \infty$,
	$\lim\nolimits_{
		\NM{\bH}{F} \rightarrow \infty}F(\cdot, \bH) 
	= \infty$.
	Thus 
	$\infty>F(\bW^{1}, \bH^1)- F(\bW^{T + 1}, \bH^{T + 1})\ge c$,
	where $c$ is a finite constant.
	Combining this with \eqref{eq:sum_seq}, when $T\rightarrow \infty$, 
	we see that a sum of infinite sequence is smaller than a finite constant.  
	This means
	the sequence
	$\{ \bW^t, \bH^t \}$ has limit points.
	Let $\{ \bar{\bW}, \bar{\bH} \}$ be a limit point, 
	we must have
	\begin{align*}
	\nabla_{\bW} F(\bar{\bW}, \bar{\bH}) = 0
	\text{\;and\;}
	\nabla_{\bH} F(\bar{\bW}, \bar{\bH}) = 0.
	\end{align*}
	By definition, this shows
	$\{ \bar{\bW}, \bar{\bH} \}$ is a critical point of \eqref{eq:fnnfn}.
	
	Next, we proceed to prove that $\bar{\bX}=\bar{\bW}\bar{\bH}^\top$ is the the critical point of \eqref{eq:mc_reg} with $r(\bX) = \rnf(\bX)$. 
		
	As shown in \cite{srebro2005maximum},  
	the nuclear norm 
	can be reformulated in terms of factorized matrices.
	Then we have 
	\begin{align*}
	\min_{\bW,\bH} & f(\bW\bH^\top)
	-\lambda\NM{\bW\bH^\top}{F} +
	\nicefrac{\lambda}{2} (\NM{\bW}{F}^2+\NM{\bH}{F}^2)
	\\
	 \ge&
	\min_{\bX} f( \bX )
	-\lambda\NM{ \bX }{F} 
	+
	\min_{\bX=\bW \bH^{\top}} \nicefrac{\lambda}{2} (\NM{\bW}{F}^2+\NM{\bH}{F}^2)
	\\
	 \ge&
	\min_{\bX} f( \bX )
	-\lambda\NM{ \bX }{F} + \lambda \NM{\bX}{*}.
	\end{align*}
	Thus,
	if $(\bar{\bW},\bar{\bH})$
	is a critical point of \eqref{eq:fnnfn},
	then $\bar{\bX}=\bar{\bW}(\bar{\bH})^\top$
	is also critical point of 
	\eqref{eq:mc_reg} with $r(\bX) = \rnf(\bX)$.
\end{proof}



\end{document}